\documentclass[final]{l4dc2026_arxiv}
\usepackage[utf8]{inputenc} 
\usepackage[T1]{fontenc}    
\usepackage{hyperref}       
\usepackage{url}            
\usepackage{booktabs}       
\usepackage{amsfonts}       
\usepackage{nicefrac}       
\usepackage{microtype}      
\usepackage{xcolor}  
\usepackage{amsmath,amsfonts,amssymb,color}
\usepackage{mathtools}
\usepackage{algorithm}
\usepackage{algpseudocode}
\usepackage{tikz}
\usepackage{lipsum}
\usepackage{wrapfig}
\usepackage{tcolorbox}
\tcbuselibrary{skins,breakable}
\usepackage{dsfont}
\usepackage{epstopdf}

\newtheorem{problem}{Problem}

\newtheorem{assumption}{Assumption}

\title[Harnessing Data from Clustered LQR Systems]{Harnessing Data from Clustered LQR Systems: Personalized and Collaborative Policy Optimization}
\usepackage{times}

\author{%
 \Name{Vinay Kanakeri} \Email{vkanake@ncsu.edu}\\
 \addr Department of Electrical and Computer Engineering, North Carolina State University
 \AND
 \Name{Shivam Bajaj} \Email{bajaj41@purdue.edu}\\
 \addr The Elmore Family School of Electrical and Computer Engineering, Purdue University%
 \AND
 \Name{Ashwin Verma} \Email{verma240@purdue.edu}\\
 \addr The Elmore Family School of Electrical and Computer Engineering, Purdue University%
 \AND
 \Name{Vijay Gupta} \Email{gupta869@purdue.edu}\\
 \addr The Elmore Family School of Electrical and Computer Engineering, Purdue University%
 \AND
 \Name{Aritra Mitra} \Email{amitra2@ncsu.edu}\\
 \addr Department of Electrical and Computer Engineering, North Carolina State University %
}

\newcommand{\abs}[1]{\lvert #1 \rvert}
\newcommand{\abslr}[1]{\left\lvert#1\right\rvert}
\newcommand{\normlr}[1]{\left\lVert#1\right\rVert}

\DeclareMathOperator*{\argmin}{\arg\!\min}
\DeclarePairedDelimiterX{\norm}[1]{\lVert}{\rVert}{#1}

\definecolor{mygreen}{rgb}{0.0, 0.5, 0.0}
\definecolor{winered}{rgb}{0.8,0,0}
\definecolor{myblue}{rgb}{0,0,0.8}

\hypersetup{
    colorlinks=true,
    linkcolor={winered},
    citecolor={myblue}
}

\begin{document}

\maketitle

\begin{abstract}%
It is known that reinforcement learning (RL) is data-hungry. To improve sample-efficiency of RL, it has been proposed that the learning algorithm utilize data from `approximately similar' processes. However, since the process models are unknown, identifying which other processes are similar poses a challenge. In this work, we study this problem in the context of the benchmark Linear Quadratic Regulator (LQR) setting. Specifically, we consider a setting with multiple agents, each corresponding to a copy of a linear process to be controlled. The agents' local processes can be partitioned into clusters based on similarities in dynamics and tasks. Combining ideas from sequential elimination and zeroth-order policy optimization, we propose a new algorithm that performs simultaneous clustering and learning to output a {\bf personalized policy} (controller) for each cluster. Under a suitable notion of cluster separation that captures differences in closed-loop performance across systems, we prove that our approach guarantees correct clustering with high probability. Furthermore, we show that the sub-optimality gap of the policy learned for each cluster scales inversely with the size of the cluster, with no additional bias, unlike in prior works on collaborative learning-based control. Our work is the first to reveal how clustering can be used in data-driven control to learn personalized policies that enjoy statistical gains from collaboration but do not suffer sub-optimality due to inclusion of data from dissimilar processes. From a distributed implementation perspective, our method is attractive as it incurs only a mild logarithmic communication overhead. 
\end{abstract}

\vspace{-3mm}
\section{Introduction}
\vspace{-1mm}
\label{sec:Intro}
The last decade or so has seen a surge of interest in \emph{model-free data-driven control}~\citep{hu2023toward}, where control laws (policies) are learned directly from data, bypassing the need to estimate the system model as an intermediate step. Although such a framework is promising, it relies on the availability of adequate data to learn high-precision policies. Unfortunately, however, data from physical processes (such as real-world robotic environments) could be \emph{scarce} and/or difficult to collect. Drawing inspiration from popular paradigms such as federated and meta-learning, some recent papers~\citep{zhang2023multi, wang2023model, toso2024meta} have attempted to mitigate this challenge by exploring the idea of combining information generated by multiple environments, where each environment represents a dynamical system with an associated cost performance metric that captures a task or a goal. The unifying theme in such papers is to learn \emph{a single common policy} that performs well across \emph{all} environments by minimizing an average-cost performance metric. When environments differ considerably in their tasks, such a single common policy might incur highly sub-optimal performance on any given environment. More fundamentally, when environments differ in dynamics, \emph{even the existence of a common stabilizing policy is unclear and difficult to verify in the absence of models.} Departing from the approach of learning a single common policy, in this paper, we ask: \emph{(When) is it possible to learn \textbf{personalized policies} in a sample-efficient way by leveraging data generated by potentially non-identical dynamical processes?} 

To formalize our study, we consider a scenario involving multiple agents that can be partitioned into distinct clusters. We assume that all agents within a given cluster interact with the same physical environment modeled as a linear time-variant (LTI) system with \emph{unknown} dynamics; furthermore, all agents within a cluster share the same quadratic cost function. However, the dynamics and cost functions across clusters can be \emph{arbitrarily different.} Thus, our setting captures both similarities and differences in dynamics and tasks. As is common in collaborative and federated learning, we allow agents to exchange information via a central aggregator. Concretely, our problem of interest is to learn a \emph{personalized} policy for each cluster that enjoys the \emph{benefits of collaboration}, i.e., we wish to show that such a policy can be learned faster (relative to a single-agent setting) by using the collective samples available within the cluster. However, this is challenging, as we explain below. 

\textbf{Challenges.} To make our setting realistic, we assume that the cluster structure is unknown a priori. Since the system models associated with the clusters are also unknown, it becomes difficult to decide how information should be exchanged between agents. In particular, care needs to be taken to avoid misclustering, since transfer of information across clusters \emph{with arbitrarily different LTI systems} can lead to the learning of destabilizing policies; thus, in our setting, \emph{more data can potentially hurt if not used judiciously.} Additional subtleties arise as the agents in our setting access only noisy zeroth-order information for both clustering and learning policies; we discuss them in Sections~\ref{sec:ProbForm} and~\ref{sec:algo}. In light of these challenges, the main \textbf{contributions} of this paper are as follows. 

$\bullet$ \textbf{Problem Formulation.} While clustering has been explored in federated learning (FL) for static supervised learning tasks~\citep{ghoshNIPs}, our work provides the first principled study of clustering in the context of model-free data-driven control, and shows how such a formalism can enable learning personalized policies in a sample-efficient manner. As part of our formulation, we identify a dissimilarity metric $\Delta$ (see~\eqref{eqn:het}) that captures differences in optimal costs between clusters. Our results reveal that a larger value of $\Delta$ leads to a faster separation of clusters. 

$\bullet$ \textbf{Novel Algorithm.} The primary contribution of this paper is the development of a novel \emph{model-free} Personalized and Collaborative Policy Optimization (PO) algorithm (Algorithm~\ref{algo:PCPO}) called \texttt{PCPO} that combines ideas from \emph{sequential elimination} in multi-armed bandits~\citep{even2006action} and policy gradient algorithms in reinforcement learning (RL)~\citep{agarwal2021theory}. The lack of prior knowledge of the cluster-separation gap $\Delta$ motivates the need for a sequential elimination strategy to identify the clusters. Moreover, since it is non-trivial to decide when to stop clustering and start collaborating, we propose an epoch-based approach that involves clustering and collaboration \emph{in every epoch} by requiring agents to maintain two separate sequences of policies: local policies used purely for sequential clustering and global policies for collaboration. Another key feature of our algorithm is that it only incurs a mild communication cost that scales logarithmically with the number of agents and samples, making it particularly appealing for distributed implementation.

$\bullet$ \textbf{Collaborative Gains.} In Theorem~\ref{thm:clustering}, we prove that with high probability, our approach leads to correct clustering, despite the absence of prior knowledge of models, cluster structure, and separation gap $\Delta.$ This result also reveals that more \emph{heterogeneity can actually aid the learning process} in that a larger $\Delta$ incurs fewer noisy function evaluations for correct clustering. Building on Theorem~\ref{thm:clustering}, our main result in Theorem~\ref{thm:collaborative_optimization} proves that by using \texttt{PCPO}, each agent can learn a near-optimal personalized policy for its own system with a sub-optimality gap that scales inversely with the number of agents in its cluster. In other words, \texttt{PCPO} prevents negative transfer of information across clusters, while ensuring sample-complexity reductions via collaboration within each cluster. To our knowledge, this is the first result to show how data from heterogeneous dynamical systems admitting a cluster structure can be harnessed to expedite the learning of personalized policies. 

Since this is a preliminary investigation of clustering in data-driven control, we restrict our attention to the canonical Linear Quadratic Regulator (LQR) formalism. That said, we anticipate that our general algorithmic template can be used in other supervised or RL problems. 

\textbf{Related Work.} To put our contributions into perspective, we discuss relevant literature below. 

$\bullet$ \textbf{Policy Gradient for LQR.} We build on the rich set of results on policy gradient methods for the LQR problem~\citep{fazel, malik2020derivative, gravell2020learning, zhang2021policy, mohammadi2019global, mohammadi2020linear, hu2023toward, moghaddam2025sample}. Generalizing these results from the single system setting to our clustered multi-system formulation introduces various nuances and challenges (outlined in Sections~\ref{sec:ProbForm} and~\ref{sec:algo}) that we address in this paper. 

$\bullet$ \textbf{Personalized Federated Learning.} We draw inspiration from the work on clustering in FL~\citep{ghoshNIPs, ghoshTIT, sattler2020} that aims to learn personalized models for groups of agents that are similar in terms of their data distributions. Despite cosmetic similarities, the specifics of our setting differ significantly in that we focus on control of dynamical systems where stability plays a crucial role; no such stability concerns arise in the FL papers above on supervised learning. 



$\bullet$ \textbf{Collaborative System Identification.} Our paper is related to a growing body of work that seeks to leverage data from multiple dynamical systems to achieve statistical gains in estimation accuracy. In this context, several papers~\citep{wang2023fedsysid, toso2023learning, fabiolti, modi2024joint, rui2025learning, tupe2025federated, xin2025learning} have explored collaborative system identification by combining trajectory data from multiple systems that share structural similarities. In particular,~\cite{toso2023learning} and~\cite{rui2025learning} assume a cluster structure like us. While the above papers focus on using collective data for an \emph{open-loop} estimation problem, namely system-identification, our work focuses instead on data-efficient closed-loop control by directly learning policies. As such, our notion of heterogeneity captures differences in closed-loop performance across clusters as opposed to similarity metrics imposed on open-loop system matrices in the papers above. 


$\bullet$ \textbf{Meta, Multi-Task, and Transfer Learning in Control.} Under the umbrella framework of meta and transfer learning, various recent papers~\citep{wang2023model, toso2024meta, aravind2024moreau, stamouli2025policy} have used PO methods to study how information from multiple LTI systems can be aggregated to learn policies that adapt across similar systems. Our formulation, which seeks to find a \emph{personalized} policy for every system, departs fundamentally from this line of work which instead aims to learn a common policy for all systems. In this regard, we note that the closely related papers of~\cite{wang2023model} and~\cite{toso2024meta} need to assume that \emph{all} the systems are sufficiently similar to admit a common stabilizing set. Even under this restrictive assumption, the results in these papers indicate that the sub-optimality gap exhibits an additive heterogeneity-induced bias term that might negate the speedups from collaboration. In contrast, our work does not require a common stabilizing policy to exist for systems across clusters. Furthermore, our personalization approach completely eliminates heterogeneity-induced biases. We also note that our approach incurs a logarithmic (in agents and samples) communication cost as opposed to the linear cost in~\cite{wang2023model}. Finally, complementary to our clustering-based approach, ideas from representation learning~\citep{zhang2023multi, guo2023imitation, lee2025regret} and domain randomization~\citep{fujinami2025domain} have also been recently used to improve data-efficiency in dynamic control tasks.  
\vspace{-4mm}
\section{Problem Formulation}
\vspace{-1mm}
\label{sec:ProbForm}
We consider a setting with $N$ agents partitioned into $H$ disjoint clusters $\{\mathcal{M}_j\}_{j \in [H]}.$ With each cluster $j \in [H]$, we associate a tuple $\mathcal{S}_j=(A_j, B_j, Q_j, R_j)$, comprising a system matrix $A_j \in \mathbb{R}^{n \times n}$, a control input matrix $B_j \in \mathbb{R}^{n \times m}$, and two positive definite matrices $Q_j \in \mathbb{R}^{n \times n}$, $R_j \in \mathbb{R}^{m \times m}$ that define the LQR cost function for cluster $j$. Each agent in cluster $j$ interacts with the \emph{same instance} of the LQR problem specified by $\mathcal{S}_j$, and aims to find a linear policy of the form $u_t = - K x_t$ that minimizes the following infinite-horizon discounted cost:
\begin{equation}
\vspace{-1mm}
C_j(K) = \mathbb{E} \left[ \sum_{t=0}^{\infty} \gamma^t \left(x^{\top}_t Q_j x_t +  u^{\top}_t R_j u_t\right)\right] \hspace{3mm} \textrm{subject to} \hspace{3mm} x_{t+1}=A_j x_t + B_j u_t + z_t,
\label{eqn:LQR}
\vspace{-1mm}
\end{equation}
where $x_0 = 0$ and $x_t$, $u_t$, and $z_t$ are the state, control input (action), and exogenous process noise, respectively, at time $t$, $\gamma \in (0,1)$ is a discount factor, and $K$ is a control gain matrix. We make the standard assumption that the pair $(A_j, B_j)$ is controllable for every $j \in [H]$. Following~\cite{malik2020derivative}, we assume that $z_t$ is sampled independently from a distribution $\mathcal{D}$, such that:
\begin{equation}
\mathbb{E}[z_t]=0, \hspace{2mm} \mathbb{E}[z_t z^{\top}_t]= I, \hspace{2mm} \textrm{and} \hspace{2mm} \Vert z_t \Vert^2_2 \leq B, \forall t,
\end{equation}
where $B > 0$ is some positive constant. For the LQR problem described in~\eqref{eqn:LQR}, it is well known~\citep{bertsekas2015dynamic} that the optimal control law is a linear feedback policy of the form $u_t = -K^*_j x_t$, where $K^*_j$ is the optimal control gain matrix for cluster $j$. When $\mathcal{S}_j$ is known, each agent in $\mathcal{M}_j$ can obtain $K^*_j$ 
by solving the discrete-time algebraic Riccati equation (DARE)~\citep{anderson2007optimal}. 

However, our interest is in the learning scenario where the system matrices $\{(A_j, B_j)\}_{j \in [H]}$ are \emph{unknown} to the agents. Even in this setting, it is known that policy optimization (PO) algorithms that treat the control gain as the optimization variable converge to the optimal policy~\citep{fazel, malik2020derivative}. The implementation of such algorithms relies on noisy \emph{trajectory rollouts} to compute estimates of policy gradients.\footnote{The notion of a rollout will be made precise later in this section.} Specifically, given $T$ independent rollouts from the tuple $\mathcal{S}_j$, each agent within $\mathcal{M}_j$ can generate a gain $\hat{K}$ such that with high probability, $C_j(\hat{K}) - C_j(K^*_j) \leq \tilde{O}(1/\sqrt{T})$~\citep{malik2020derivative}. Our goal is to investigate whether this sample-complexity bound can be improved by leveraging the cluster structure in our problem.  

To achieve potential gains in sample-complexity via collaboration, we allow the agents to communicate via a central server, and make the following assumption that is common in the literature on collaborative/federated learning~\citep{konevcny, mcmahan}. 
\vspace{-0.5mm}
\begin{assumption}
\label{ass:independence}
The noise processes across agents are statistically independent, i.e., for all $i_1, i_2 \in [N]$ such that $i_1 \neq i_2$, the noise stochastic processes $\{z^{(i_1)}_t\}$ and $\{z^{(i_2)}_t\}$ are mutually independent. Here, with a slight overload of notation, we use $\{z^{(i)}_t\}$ to denote the noise process for agent $i \in [N]$. 
\end{assumption}
\vspace{-0.5mm}
Although the above assumption suggests that exchange of information between agents can accelerate the learning of an optimal policy,  collaboration is complicated by the \emph{heterogeneity} among clusters, due to the difference in system dynamics $(A_j,B_j)$ and in task objectives $(Q_j,R_j)$. To capture such heterogeneity across clusters, we take inspiration from the notion of ``cluster separation gaps" in supervised learning~\citep{ghoshTIT, su2022global, sattler2020}, and introduce 
\begin{equation}
    \Delta \coloneqq \min_{j_1, j_2 \in [H]:j_1 \neq j_2} \vert C_{j_1}(K^*_{j_1}) - C_{j_2}(K^*_{j_2}) \vert.
\label{eqn:het}
\end{equation}

We assume a non-zero separation across clusters, i.e., $\Delta >0$. While one can certainly formulate alternative notions of heterogeneity, we will show (in Theorem~\ref{thm:collaborative_optimization}) that our metric of cluster-dissimilarity, as captured by $\Delta$ in~\eqref{eqn:het}, can be suitably exploited to separate clusters and learn personalized policies for each cluster. In particular, our results reveal that \emph{heterogeneity can be helpful:} a larger value of $\Delta$ leads to faster cluster separation using fewer rollouts. With the above ideas in place, we are now ready to formally state our problem of interest. For each agent $i \in [N]$, let us use $\sigma(i)$ to represent the index of the cluster to which it belongs. 
 \begin{tcolorbox}[colback=blue!5!white,colframe=blue,breakable]
\begin{problem} 
\label{prob:ClustLQR}
(\textbf{Clustered LQR Problem}) Let $\delta \in (0,1)$ be a given failure probability. Suppose every agent $i \in [N]$ has access to $T$ independent rollouts from its corresponding system $\mathcal{S}_{\sigma(i)}.$ Develop an algorithm that returns $\{\hat{K}_i\}_{i \in [N]}$ such that with probability at least $1-\delta$, the following is true $\forall i \in [N]$: 
$$C_{\sigma(i)} (\hat{K}_i) - C_{\sigma(i)}(K^*_{\sigma(i)}) \leq \tilde{O}\left( \frac{1}{\sqrt{| \textcolor{winered}{\mathcal{M}_{\sigma(i)}}| T}} \right).$$
\end{problem}
\end{tcolorbox} 
In simple words, our goal is to come up with an algorithm that generates a personalized control policy for every agent that benefits from the collective information available within that agent's cluster. This is quite non-trivial due to the following technical \textbf{challenges}. 

$\bullet $ In our setting, the system dynamics and the cluster identities are both unknown a priori. Thus, our problem requires learning the cluster identities and optimal policies simultaneously. 

$\bullet$ The clustering process is complicated by two  main issues. First, the information used for clustering is based on noisy function evaluations that are insufficient for estimating the system models, ruling out system-identification-based approaches in~\cite{toso2023learning} and~\cite{rui2025learning}. Thus, we need to develop a \emph{model-free} clustering algorithm. Second, the minimum separation gap $\Delta$ in~\eqref{eqn:het} is assumed to be unknown, ruling out the possibility of simple one-shot clustering approaches. 

$ \bullet $ Unlike supervised learning problems in~\cite{ghoshNIPs} and~\cite{su2022global} where misclustering only introduces a bias due to heterogeneity, the price of misclustering can be more severe in our control setting. In particular, since the system tuples across clusters are allowed to be \emph{arbitrarily different}, transfer of information across clusters can lead to destabilizing policies. 

In the next section, we will develop the \texttt{PCPO} algorithm that addresses the above challenges and solves Problem~\ref{prob:ClustLQR}, while incurring only a logarithmic (with respect to the number of agents and rollouts) communication cost. In preparation for the next section, we now define the notion of a \emph{rollout}  and a zeroth-order gradient estimator. Given a policy $K$, a rollout for an agent $i \in \mathcal{M}_j$ yields a \emph{noisy sample} of the infinite-horizon trajectory cost, defined as:
\begin{equation}
\label{eqn:rollout}
    C_j(K; \mathcal{Z}^{(i)}) = \sum_{t=0}^{\infty} \gamma^t \left(x^{\top}_t Q_j x_t +  u^{\top}_t R_j u_t\right), \hspace{1mm} \textrm{where} \hspace{1mm} x_{t+1}=A_j x_t + B_j u_t + z_t, u_t = -Kx_t,  
\end{equation}
 $x_0 = 0$ and $\mathcal{Z}^{(i)} = \{z_t^{(i)}\}$. We will interpret each rollout as a sample. Using such noisy function evaluations for a policy $K$ run by an agent $i \in \mathcal{M}_j$, we define the $M$-minibatched zeroth-order gradient estimator with a smoothing radius $r$, as follows~\citep{fazel, malik2020derivative}:
\begin{equation}
\label{eqn:zero-order-grad}
    g_i(K) \coloneqq \frac{1}{M} \sum_{k=1}^{M} C_j(K + rU_k; \mathcal{Z}_k^{(i)})\left(\frac{D}{r}\right)U_k,
\end{equation}
where $D = mn$, $U_k$ is drawn independently from a uniform distribution over matrices with unit Frobenius norm, and $\mathcal{Z}_k^{(i)}$ are independent copies of $\mathcal{Z}^{(i)}$ for all $k \in [M]$. In the sequel, for an agent $i \in \mathcal{M}_j$, we use the shorthand $\texttt{ZO}_i(K, M, r)$ to refer to the $M$-minibatched zeroth-order gradient estimator at policy $K$ with smoothing radius $r$, as defined in \eqref{eqn:zero-order-grad}. We use the notation $c_{(p, \_)}$ to denote problem-parameter-dependent constants and provide their expressions in Appendix~\ref{app:properties} of \cite{L4DC26}. We make the standard assumption \citep{fazel, malik2020derivative} that each agent $i$ has access to an initial controller $K_i^{(0)}$ that lies within its respective stabilizing set: $\{K \in \mathbb{R}^{m \times n}: \rho(A_{\sigma(i)} - B_{\sigma(i)}K) < 1\}$, where $\rho(X)$ is the spectral radius of a matrix $X \in \mathbb{R}^{n \times n}$.
\vspace{-3mm}
\section{Description of the Algorithm}
\vspace{-1mm}
\label{sec:algo}
In this section, we present our proposed algorithm, Personalized and Collaborative Policy Optimization (\texttt{PCPO}) (Algorithm~\ref{algo:PCPO}), which effectively addresses Problem~\ref{prob:ClustLQR} by carefully accounting for its inherent challenges. Since the cluster separation gap $\Delta$ is unknown, we propose a \emph{sequential elimination} strategy to identify the correct clusters. While the idea of sequential elimination has been explored in the context of multi-armed bandits \citep{even2006action}, we show that a similar approach can be used to effectively cluster LTI systems that satisfy the heterogeneity metric defined in \eqref{eqn:het}. The algorithm proceeds in epochs (indexed by $l$) of increasing duration, where in each epoch, each agent $i$ updates two sequences: a local sequence, $\{X_i^{(l)}\}_{l\geq 0}$, and a global sequence, $\{\hat{K}_i^{(l)}\}_{l \geq 0}$. The local sequence is updated using the zeroth-order policy optimization algorithm in \cite{fazel, malik2020derivative}, and is used exclusively for clustering the agents.  The global sequence is updated by aggregating the gradient estimates from all the agents within an appropriately defined neighborhood set. After each epoch, the neighborhood sets are updated based on the concentration of the estimated cost around the optimal cost. As the number of rollouts increases across epochs, this concentration becomes tighter, hence pruning out misclustered agents over successive epochs. The various components of the algorithm are succinctly captured in Figure~\ref{fig:algo}. In Section~\ref{sec:results}, we show that the neighborhood sets eventually converge to the correct clusters with high probability, after which the global sequence enjoys the collaborative gains without any heterogeneity bias. In what follows, we elaborate on the rationale behind the design of the various components of \texttt{PCPO}.

\begin{figure}[t]
\centering
  \includegraphics[scale =0.425]{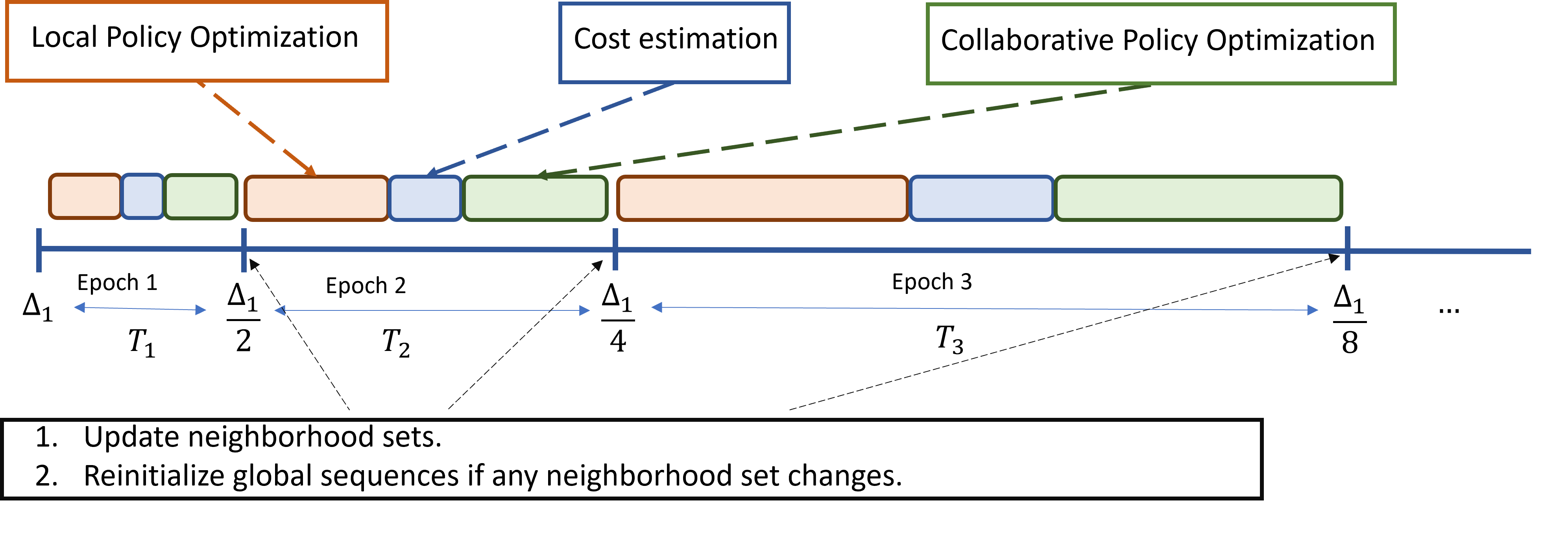}
  \vspace{-6mm}
  \caption{Illustration of the epoch-based structure of \texttt{PCPO}, where each epoch involves three key steps: local policy optimization (PO), cost estimation, and global PO.}
  \label{fig:algo}
\vspace{-8mm}
\end{figure} 
\begin{algorithm}[t]
\caption{Personalized and Collaborative Policy Optimization (\texttt{PCPO})}
\label{algo:PCPO}  
\begin{algorithmic}[1] 
\State \textbf{Initialization:} $\Delta_0$; $\forall i \in [N]$, $\hat{K}_i^{(0)} \leftarrow K_{i}^{(0)}, X_i^{(0)} \leftarrow K_{i}^{(0)},  \mathcal{N}_i^{(0)} \leftarrow [N]$.
\State \textbf{For} \hspace{0.25mm} {$l = 1, 2, \ldots, $} \hspace{0.25mm}
\State \quad \textbf{At Each Agent $i$:} $\Delta_l \leftarrow \frac{\Delta_{l-1}}{2},$ $\delta_l \leftarrow \frac{\delta}{2l^2}, \eta \leftarrow c_{(p, 1)},$ $R_l \leftarrow c_{(p, 2)} \log\left(\frac{c_{(p, 3)}N}{\Delta_l^2}\right)$
\Statex \quad $M_l \leftarrow \frac{c_{(p, 4)}}{\Delta_l^2}\log\left(\frac{8DNR_l}{\delta_l}\right),$ $\tilde{r}_l \leftarrow \left(\frac{c_{(p, 5)}}{\sqrt{M_l}} \sqrt{\log\left(\frac{8DNR_l}{\delta_l}\right)}\right)^{1/2}, r_l^{(\text{loc})} \leftarrow \min\{c_{(p, 6)}, \tilde{r}_l\}.$
\State \quad \textcolor{myblue}{Local Policy Optimization}: $X_i^{(l)} \leftarrow  \texttt{localPO}(X_i^{(l-1)}, M_l, R_l, r_l^{(\text{loc})}).$
\State \quad \textcolor{myblue}{Cost estimation}: $\hat{C}_{\sigma(i)}(X_i^{(l)}) \leftarrow \frac{1}{M_l} \sum_{j=1}^{M_l} C_{\sigma(i)}(X_i^{(l)}, \mathcal{Z}_j^{(i)})$. 
\State \quad Initialize $Y_i^{(0)} \leftarrow \hat{K}_i^{(l-1)}$ and set $r_{(i, l)}^{(\text{global})} \leftarrow \min\left\{c_{(p, 6)}, \frac{\tilde{r}_l}{\abs{\mathcal{N}_i^{(l-1)}}^{1/4}}\right\}$ \Comment{\textcolor{myblue}{For collaborative PO}}
\State \quad \textbf{For} \hspace{0.25mm} {$k = 0, 1, \ldots, R_l-1$} \hspace{0.25mm}
\State \quad \quad \textbf{At Each Agent $i$:} Transmit $g_i(Y_i^{(k)}) \leftarrow \texttt{ZO}_i(Y_i^{(k)}, M_l, r_{(i, l)}^{(\text{global})})$ to the Server.
\State \quad \quad \textbf{At Server:} Compute and transmit the averaged gradient estimate as follows:
\begin{equation}
    G_i \leftarrow \frac{1}{\vert\mathcal{N}_i^{(l-1)}\vert} \sum_{j \in \mathcal{N}_i^{(l-1)}} g_j(Y_j^{(k)}).
    \label{eqn:averaged_grad}
\end{equation}
\State \quad \quad \textbf{At Each Agent $i$:} $Y_i^{(k+1)} \leftarrow Y_i^{(k)} - \eta G_i.$ \Comment{  
\textcolor{myblue}{Global policy update via collaboration}}
\State \quad \textbf{End For}
\State \quad \textbf{At Each Agent $i$:} Update $\hat{K}_i^{(l)} \leftarrow Y_i^{(R_l)}.$ Transmit $X_i^{(l)}, \hat{C}_{\sigma(i)}(X_i^{(l)})$ to the server.
\State \quad \textbf{At Server:} \quad Update the neighborhood set as follows: \Comment{ \textcolor{myblue}  {Sequential elimination}}
\begin{equation}
    \mathcal{N}_i^{(l)} \leftarrow \{j \in \mathcal{N}_i^{(l-1)} : \abslr{\hat{C}_{\sigma(j)}(X_j^{(l)}) - \hat{C}_{\sigma(i)}(X_i^{(l)})} \leq \Delta_l/2 \} 
    \label{eqn:neighborhood_update}
\end{equation}
\State \quad \textbf{If} $\mathcal{N}_i^{(l)} \neq \mathcal{N}_i^{(l-1)}$ for some $i \in [N]$: \Comment{ \textcolor{myblue}{Reinitialization to ensure stability}}
\State \quad \quad For all agents $i \in [N]$, update and transmit $\hat{K}_i^{(l)}$ as follows:
\begin{equation}
  \hat{K}_i^{(l)} \leftarrow \argmin_{\{X_j^{(l)}: j \in \mathcal{N}_i^{(l)}\}} \{\hat{C}_{\sigma(j)}(X_j^{(l)}): j \in \mathcal{N}_i^{(l)}\}.
  \label{eqn:reinitialization}
\end{equation}
\State \textbf{End For}
\end{algorithmic}
\end{algorithm}

\textbf{Building intuition.} Correctly identifying the agents' clusters is crucial to reap any potential benefits from collaboration, as collaborating with agents from a different cluster can lead to destabilizing policies. In this regard, the major difficulty arises from the fact that the cluster separation gap $\Delta$ in \eqref{eqn:het} is unknown a priori. To appreciate the associated challenges, as a thought experiment, let us consider a simpler case where $\Delta$ is known. Under this scenario, each agent can locally run policy optimization to obtain a policy in a sufficiently close neighborhood of the optimal policy. Then, each agent can evaluate the cost at this policy and ensure that it is concentrated around the optimal cost. If the neighborhood radius, which depends on $\Delta$, is carefully chosen, it is easy to see that such a one-shot approach leads to correct clustering. However, when $\Delta$ is unknown, one-shot clustering may no longer work, motivating the \emph{sequential clustering} idea in our proposed PCPO Algorithm.

\textbf{Sequential elimination.} In each epoch $l$ of \texttt{PCPO}, for every agent $i \in [N]$, the server maintains a neighborhood set $\mathcal{N}^{(l)}_i$ as an estimate of the true cluster $\mathcal{M}_{\sigma(i)}$. All such neighborhood sets are initialized from the set of all agents and sequentially pruned over epochs. For pruning, we start with an initial estimate of $\Delta$, denoted by $\Delta_0$, and update it by halving its value at the beginning of each epoch $l$ to obtain $\Delta_l$ (Line 3 of Algo.~\ref{algo:PCPO}). Our goal is to ensure that for all agents, the estimated cost in epoch $l$ is in the $\Delta_l/4$ neighborhood of its optimal cost. We achieve this in a two-step process, where we first perform local policy optimization to obtain a policy that is $\Delta_l/8$-suboptimal (Line 4). The \texttt{localPO} subroutine performs $M_l$-minibatched policy optimization for $R_l$ iterations starting with a controller $X_{i}^{(l-1)}$. In every iteration $t \in [R_l]$, $X_{(i, t)}$, the $t$-th sub-iterate of \texttt{localPO}, is updated as $X_{(i, t+1)} \leftarrow X_{(i, t)} - \eta g_i(X_{(i, t)})$, where $g_i(X_{(i, t)}) = \texttt{ZO}_i(X_{(i, t)}, M_l, r_l^{(\text{loc})})$. Then, we estimate the cost at the policy $X_{i}^{(l)}$ obtained from \texttt{localPO} with an error tolerance of $\Delta_l/8$ using $M_l$ rollouts (Line 5). Having achieved the desired cost-estimation accuracy of $\Delta_l/4$ for all agents, we prune the neighborhood sets according to \eqref{eqn:neighborhood_update} in Line 13. Eventually, as $\Delta_l \leq \Delta$, which happens in $O(\log(\Delta_0/\Delta))$ epochs, correct clustering takes place, as elaborated in the next section. 

\textbf{Local and Global Sequences.} To motivate the need for maintaining two sequences in \texttt{PCPO}, let us again consider the case where $\Delta$ is known, where it would suffice for the agents to only maintain a single sequence to run local PO until clustering, as discussed earlier in the one-shot clustering scheme. After the clusters are identified, the same sequence can be used for collaboration. In our setting, however, although the neighborhood sets eventually converge to the correct clusters in logarithmic number of epochs with respect to $1/\Delta$, the number of such epochs cannot be determined a priori without knowledge of $\Delta$, making it difficult to decide when to initiate collaboration. Furthermore, with a single sequence, collaborating with misclustered agents can lead to an undesirable scenario \emph{where the sequence used to cluster is itself contaminated due to misclustering.}  \texttt{PCPO} carefully navigates this difficulty by maintaining two sequences of policies at each agent. The local sequence, $\{X_i^{(l)}\}_{l\geq 0}$, is used purely for clustering, and the global sequence, $\{\hat{K}_i^{(l)}\}_{l\geq 0}$, is updated by aggregating gradients from agents within the neighborhood set $\mathcal{N}^{(l-1)}_i$ from the previous epoch $l-1$; see~\eqref{eqn:averaged_grad} in Line 9.

\textbf{Logarithmic communication.} Since both local and global PO are performed for $R_l$ iterations with $M_l$ rollouts per iteration, the overall sample complexity per epoch is $T_l = 2R_lM_l + M_l$, where the additional $M_l$ rollouts are due to the cost estimation step. Each iteration of global PO proceeds as follows. First, for every agent, the server combines the minibatched zeroth-order gradient estimates as per \eqref{eqn:averaged_grad}. Then, the agents update the iterates using the averaged gradient with an appropriately chosen but fixed step size $\eta$ (Line 10). Since the above essentially incurs $O(R_l)$ communication steps per epoch, the total communication complexity of \texttt{PCPO} is $O(R_l \times \textrm{Number of Epochs})$. Based on our choice of parameters, both objects in the above product are logarithmic in the number of agents $N$, the gap $1/\Delta$, and the number of total rollouts per agent, namely $T$. 

\textbf{Note on reinitialization.} Since the server averages gradients in every epoch based on the neighborhood set, agents inevitably collaborate across clusters until correct clustering is achieved. This can lead to destabilizing policies in the global sequence. To mitigate this, at the end of each epoch, we reinitialize the global policy sequences for all agents whenever any neighborhood set is updated, as specified in \eqref{eqn:reinitialization}. In the next section, we establish in Theorem~\ref{thm:clustering} that the neighborhood sets eventually converge to the correct clusters and cease to update. Consequently, reinitialization ensures that the global sequences of agents within the same cluster evolve identically and achieve collaborative gains once the correct clusters are identified.
\vspace{-3mm}
\section{Main Results}
\vspace{-1mm}
\label{sec:results}
Our main results concern the two key components of the PCPO algorithm: (i) identifying the correct clusters via sequential elimination, and (ii) performing collaborative policy optimization with logarithmic communication. The following theorem captures the clustering component.

\begin{theorem} \label{thm:clustering} (\textbf{Clustering with sequential elimination}) Define $L = \min\{l \in 1, 2, \ldots : \Delta_l \leq \Delta/2\}.$ Given a failure probability $\delta \in (0, 1)$, with probability at least $1 - \delta/2$, the following statements concerning the neighborhood sets from the \texttt{PCPO} algorithm hold for every agent $i \in [N]$:
\begin{enumerate}
    \item In every epoch $l,$ we have $ \mathcal{M}_{\sigma(i)} \subseteq \mathcal{N}_i^{(l)}$.
    \item For any epoch $l$ such that $l \geq L$, we have $ \mathcal{M}_{\sigma(i)} = \mathcal{N}_i^{(l)}$.
\end{enumerate}
\end{theorem}

\textbf{Discussion.} The key technical contribution of Theorem \ref{thm:clustering} lies in showing that the sequential elimination strategy successfully lets the neighborhood sets converge to the correct clusters with high probability. In particular, we show that the true clusters are included in the neighborhood sets for all agents in each epoch, and there exists an epoch $L$ after which the neighborhood sets have converged to the true clusters and remain fixed for all subsequent epochs $(l \geq L)$. Later in this section, we provide a proof sketch and discuss how these claims follow from the local policy optimization and the cost estimation step, relying on our notion of the cluster separation gap as defined in \eqref{eqn:het}. The following theorem captures the key result concerning the collaborative optimization part in \texttt{PCPO}.

\begin{theorem} \label{thm:collaborative_optimization} (\textbf{Collaborative Policy Optimization}) Let the failure probability be $\delta \in (0, 1)$. Define $L = \min\{l \in 1, 2, \ldots : \Delta_l \leq \Delta/2\}$ and let $\bar{L}$ denote the last epoch. If the number of rollouts per agent satisfies $T \geq \tilde{O}(1/\Delta^2)$, and Assumption~\ref{ass:independence} holds, then $\bar{L} > L$ and $\hat{K}_i^{(\bar{L})}$ satisfies the following with probability at least $1 -\delta$ for every agent $i \in [N]$:
\begin{equation}
\label{eqn:collaborativePO}
\begin{aligned}
    C_{\sigma(i)} (\hat{K}_i^{(\bar{L})}) - C_{\sigma(i)}(K^*_{\sigma(i)}) &\leq O\left(\frac{c_{(p, 7)} \sqrt{\log \left(\frac{8DNT}{\delta}\right)}}{\sqrt{T\abs{\mathcal{M}_{\sigma(i)}}}} \right).
\end{aligned}
\end{equation}  
\end{theorem}

\textbf{Discussion.} It was shown in \cite{malik2020derivative} that zeroth-order policy optimization provides a $\tilde{O}(1/\sqrt{T})$ suboptimal policy using $T$ rollouts for a single system LQR problem. In contrast, \cite{wang2023model} showed collaborative gains for a federated LQR setting while incurring additive bias terms that depend on the heterogeneity gap. Moreover, the results in \cite{wang2023model} apply only to systems with bounded heterogeneity. Theorem~\ref{thm:collaborative_optimization} bridges this gap by showing that collaborative gains (as evidenced by the $\sqrt{|\mathcal{M}_{\sigma(i)}|}$-factor speedup for each agent $i$ in~\eqref{eqn:collaborativePO}) can be achieved without any additive bias through careful cluster identification and collaboration exclusively within clusters. Furthermore, these gains hold for systems that can be \emph{arbitrarily different}, as long as their optimal costs are separated according to \eqref{eqn:het}.  

\begin{corollary} \label{cor:log_complexity} (\textbf{Logarithmic communication complexity}) The \texttt{PCPO} algorithm guarantees a logarithmic communication complexity with respect to the total number of rollouts $T$, the number of agents $N$, and the inverse of the separation gap $1/\Delta$.
\end{corollary}

In the following, we provide proof sketches for both Theorem \ref{thm:clustering} and Theorem \ref{thm:collaborative_optimization}, while deferring the detailed proofs to \cite{L4DC26}. The statements made in the proof sketches are probabilistic in nature. However, to keep the exposition simpler, we omit specifying the success/failure probability of the statements and refer the readers to \cite{L4DC26} for such details.

\textbf{Proof sketch for Theorem \ref{thm:clustering}.} We start by showing that the estimated cost for each agent is in the $\Delta_l/4$ neighborhood of its optimal cost, i.e., in each epoch $l$, for each agent $i$, we prove that 
\vspace{-2mm}
\begin{equation}
|\hat{C}_{\sigma(i)}(X^{(l)}_i) - C_{\sigma(i)}(K^*_{\sigma(i)})| \leq \Delta_l/4.
\label{eqn:localestacc}
\vspace{-1mm}
\end{equation}
Note that for any agent $j \in \mathcal{M}_{\sigma(i)}$, since $C_{\sigma(i)}(K^*_{\sigma(i)}) = C_{\sigma(j)}(K^*_{\sigma(j)})$, in light of~\eqref{eqn:localestacc}, it is apparent that such an agent will pass the requirement in~\eqref{eqn:neighborhood_update}, and hence, never be eliminated from the neighborhood sets of agent $i$. This explains the first claim of Theorem~\ref{thm:clustering}. As for the second claim, note that for any $j \notin \mathcal{M}_{\sigma(i)}$, in light of our dissimilarity metric $\Delta$ in~\eqref{eqn:het}, $|C_{\sigma(i)}(K^*_{\sigma(i)}) - C_{\sigma(j)}(K^*_{\sigma(j)})| \geq \Delta.$ Now for an epoch $l$ such that $\Delta_l \leq \Delta/2$, the above inequality can be combined with that in~\eqref{eqn:localestacc} to see that $|\hat{C}_{\sigma(i)}(X^{(l)}_i) - \hat{C}_{\sigma(j)}(X^{(l)}_j)| \geq 3 \Delta/4$, violating the requirement for inclusion in~\eqref{eqn:neighborhood_update}. It remains to establish~\eqref{eqn:localestacc}, which follows from two guarantees: (i) $|\hat{C}_{\sigma(i)}(X^{(l)}_i) - C_{\sigma(i)}(X^{(l)}_i)| \leq \Delta_l/8,$ and (ii) $|C_{\sigma(i)}(X^{(l)}_i) - C_{\sigma(i)}(K^*_{\sigma(i)})| \leq \Delta_l/8$. The second guarantee follows from an analysis of the local PO sub-routine in Line 4 of Algo.~\ref{algo:PCPO}, drawing on~\cite{malik2020derivative}; the first follows from analyzing the cost estimation step in Line 5 based on a simple Hoeffding bound. 

\textbf{Proof sketch for Theorem \ref{thm:collaborative_optimization}.} We prove Theorem \ref{thm:collaborative_optimization} by conditioning on the event where the claims made in Theorem \ref{thm:clustering} hold. As agents collaborate exclusively within their respective clusters and the reinitialization step synchronizes their global sequences after correct clustering, the gradient estimate obtained by averaging (see \eqref{eqn:averaged_grad}) estimates of agents within a cluster enjoys a \emph{variance reduction effect} under Assumption~\ref{ass:independence}. Using this, and the fact that the LQR cost satisfies a $\phi$-smoothness and $\mu$-PL condition locally~\citep{fazel, malik2020derivative}, we establish that in each iteration $k$ of the final epoch $\bar{L}$, the following recursion holds with probability $1 - \delta'$:
\vspace{-2mm}
\begin{equation}
    S_{k+1} \leq \left( 1 - \frac{\eta \mu}{4} \right) S_k + 3\eta\left(\underbrace{\frac{c_{(p, 8)}^2 D^2}{(r_{(i, \bar{L})}^{(\text{global})})^2\abs{\mathcal{M}_{\sigma(i)}}M_{\bar{L}}} \log\left(\frac{2D}{\delta'}\right)}_{s_1} + \underbrace{\phi^2(r_{(i, \bar{L})}^{(\text{global})})^2}_{s_2}\right),
\vspace{-1mm}
\end{equation}
where $S_k \coloneqq C_{\sigma(i)}(Y_i^{(k)}) - C_{\sigma(i)}(K_{\sigma(i)}^*)$. The term $s_1$ is due to the concentration of the minibatched gradient estimate around the gradient of a smoothed cost defined in Appendix \ref{app:properties} of \cite{L4DC26}; this is the term that benefits from collaboration. The term $s_2$ captures the bias that arises when estimating gradients from noisy function evaluations. To ensure that this bias term does not negate the collaborative speedup in $s_1$, we choose the smoothing radius $r_{(i, \bar{L})}^{(\text{global})}$ to minimize the sum $s_1 + s_2$. Unrolling the recursion for $R_l$ iterations (with the choice of $R_{l}$ in \texttt{PCPO}) provides the per epoch convergence with rate $\tilde{O}\left(1/\sqrt{\abs{\mathcal{M}_{\sigma(i)}}M_{\bar{L}}}\right)$. Finally, using the fact that $M_l$ increases exponentially with epochs, we establish that $M_{\bar{L}} = \tilde{\Omega}(T)$.   

\vspace{-3mm}
\section{Conclusion}
\vspace{-1mm}
We developed a novel clustering-based approach for learning personalized control policies using data from heterogeneous dynamical processes. As future work, we will explore (i) alternative measures of dissimilarity across systems, (ii) more general dynamical processes, and (iii) online settings. 
\bibliography{refs.bib}
\newpage

\appendix
\section{Properties of the LQR problem}
\label{app:properties}
\textbf{Notation.} For matrices $A\in \mathbb{R}^{m \times n}$ and $B \in \mathbb{R}^{m \times n}$, we use $\norm{A}$ to denote the Frobenius norm of $A$ which is defined as $\norm{A} = \sqrt{\text{trace}(A^\top A)}$. We use $\langle A, B \rangle$ to denote the Frobenius inner-product defined as $\langle A, B \rangle = \text{trace}(A^\top B)$.

In this section, we discuss some of the properties of the LQR cost that were established in \cite{fazel, malik2020derivative}. In particular, the LQR cost in \eqref{eqn:LQR} is locally Lipschitz, locally smooth, and enjoys a gradient-domination property over the set of stabilizing controllers. These key properties aid in the convergence analysis of the model-free policy gradient algorithm, and we use them in the proofs of Theorem~\ref{thm:clustering} in Appendix~\ref{app:proof_thm1} and Theorem~\ref{thm:collaborative_optimization} in Appendix~\ref{app:proof_thm2}. However, to show the convergence, it is crucial to ensure that the policy gradient iterates always lie within a restricted subset of the stabilizing set with high probability. In \cite{malik2020derivative}, such a restricted set is chosen based on the initial suboptimality gap. In our setting, given access to a set of initial stabilizing controllers for all agents, $\{K_i^{(0)}\}_{i \in [N]}$, and our initial guess for the cluster separation gap, $\Delta_0$, we define $\tilde{\Delta}_0 \coloneqq 
\max\{\max_{i \in [N]}(C_{\sigma(i)}(K_i^{(0)}) - C_{\sigma(i)}(K_{\sigma(i)}^*)), \Delta_0\}$, and the restricted sets as follows for all $j \in [H]$:
\begin{equation}
\label{eqn:good_set}
    \mathcal{G}_j^{0} \coloneqq \{K \in \mathbb{R}^{m \times n} : C_j(K) - C_j(K_j^*) \leq 10 \tilde{\Delta}_0\}.
\end{equation}

\textbf{Properties of the LQR cost.} For each system $j \in [H]$, on the restricted domain $\mathcal{G}_j^{0}$, \cite{malik2020derivative} showed that the local properties hold uniformly, i.e., $\exists \phi_j > 0, \lambda_j > 0, \rho_j > 0,$ such that the LQR cost in \eqref{eqn:LQR} is $(\lambda_j, \rho_j)$-locally Lipschitz and $(\phi_j, \rho_j)$-locally smooth for all policies $K \in \mathcal{G}_j^0$. Furthermore, it is known that the LQR cost satisfies the PL (gradient-domination) condition for all policies in the stabilizing set (see Lemma 3 of \cite{malik2020derivative}). Denoting the parameter for the PL condition for system $j$ by $\mu_j > 0$, we define $\mu \coloneqq \min\{\mu_1, \mu_2, \ldots, \mu_H\}$. Similarly, defining $\phi \coloneqq \max\{\phi_1, \phi_2, \ldots, \phi_H\}$, $\lambda \coloneqq \max\{\lambda_1, \lambda_2, \ldots, \lambda_H\}$, and $\rho \coloneqq \min\{\rho_1, \rho_2, \ldots, \rho_H\}$, we can ensure that for every system $j \in [H]$, the cost in \eqref{eqn:LQR} is $(\lambda, \rho)$-locally Lipschitz and $(\phi, \rho)$-locally smooth for all policies in their respective restricted sets $\mathcal{G}_j^0$, and $\mu$-PL in their respective stabilizing sets. The following lemmas from \cite{malik2020derivative} capture these properties. 

\begin{lemma} \label{lem:local_lipschitz} (\textbf{LQR cost is locally Lipschitz}). For any system $j \in [H]$, given a pair of policies $(K, K') \in (\mathcal{G}_j^0 \times \mathcal{G}_j^0)$, if $\norm{K - K'} \leq \rho$, we have
\begin{equation}
    \label{eqn:local_lipschitz}
    \abslr{C_j(K) - C_j(K')} \leq \lambda \norm{K - K'}.
\end{equation}
\end{lemma}

\begin{lemma} \label{lem:local_smooth} (\textbf{LQR cost has locally Lipschitz gradients.}) For any system $j \in [H]$, given a pair of policies $(K, K') \in (\mathcal{G}_j^0 \times \mathcal{G}_j^0)$, if $\norm{K - K'} \leq \rho$, we have
\begin{equation}
    \label{eqn:local_lipschitz}
    \normlr{\nabla C_j(K) - \nabla C_j(K')} \leq \phi \norm{K - K'}.
\end{equation}
\end{lemma}

\begin{lemma} \label{lem:pl} (\textbf{LQR cost satisfies PL.}) For any system $j \in [H]$, given a stable policy $K$, we have 
\begin{equation}
\label{eqn:PL_condition}
    \norm{\nabla C_j(K)}^2 \geq \mu (C_j(K) - C_j(K_j^*)).
\end{equation}
\end{lemma}

\textbf{Smoothed cost and the properties of the gradient estimate.} The smoothed cost with a radius $r$ for a system $j \in [H]$ is defined as $C_{j, r}(K) \coloneqq \mathbb{E}[C_j(K + rv)]$, where $v$ is uniformly distributed over all matrices in $\mathbb{R}^{m \times n}$ with the Frobenius norm of at most 1. It is shown in \cite{fazel, malik2020derivative} that the zeroth-order gradient estimate $g_i(\cdot)$ as defined in \eqref{eqn:zero-order-grad} for an agent $i$ is an unbiased estimator of the gradient of the smoothed cost. In particular, for all systems $j \in [H]$ and all agents $i \in \mathcal{M}_j$, we have the following properties for all $K \in \mathcal{G}_j^0$ and $r \in (0, \rho)$: 
\begin{align}
    \mathbb{E}[g_i(K)] &= \nabla C_{\sigma(i), r}(K) \label{eqn:unbiased_grad}\\
    \norm{\nabla C_{j, r}(K) - \nabla C_{j}(K)} &\leq \phi r. \label{eqn:bounded_bias}
\end{align}

Furthermore, it is known that the noisy rollout cost $C_j(K + rU; \mathcal{Z}^{(i)})$ is uniformly bounded  $ \forall K \in \mathcal{G}_j^0$, $ \forall r \in (0, \rho)$ and all $U \in \mathbb{R}^{m \times n}$ with $\norm{U} = 1$. In other words, there exists $G_\infty^{(j)} \geq 0$ such that $C_j(K + rU; \mathcal{Z}^{(i)}) \leq G_\infty^{(j)}$ (see Lemma 11 of \cite{malik2020derivative} for the expression of $G_\infty^{(j)}$.) Let us define $G_\infty \coloneqq \max\{G_\infty^{(1)}, G_\infty^{(2)}, \ldots, G_\infty^{(H)}\}$. Hence, for any agent $i \in [N]$, we have the following $\forall K \in \mathcal{G}_j^0$, $\forall r \in (0, \rho)$ and all $U \in \mathbb{R}^{m \times n}$ with $\norm{U} = 1$: 
\begin{equation}
\label{eqn:uniformly_bounded_noisy_cost}
    C_{\sigma(i)}(K + rU; \mathcal{Z}^{(i)}) \leq G_\infty.
\end{equation}

We then have the following concentration result that will prove useful in establishing ``variance-reduction" effects. 
\begin{lemma} \label{lem:grad_concentration} (\textbf{Concentration of the zeroth-order gradient estimates}). For any system $j \in [H]$, given a policy $K \in \mathcal{G}_j^0$, a smoothing radius $r \in (0, \rho)$ and a failure probability $\delta' \in (0, 1)$, the following holds for the $M$-minibatched zeroth-order gradient estimate of an agent $i \in \mathcal{M}_{j}$ with probability at least $1 - \delta'$: 
\begin{equation}
\label{eqn:grad_concentration}
    \norm{g_i(K) - \nabla C_{j, r}(K)} \leq \frac{\left(G_\infty +  \lambda\frac{\rho}{D} + \phi \frac{\rho^2}{D}\right)D}{r\sqrt{M}} \sqrt{\log\left(\frac{2D}{\delta'}\right)}.
\end{equation}
\end{lemma}
\begin{proof} We have $\norm{g_i(K) - \nabla C_{j, r}(K)} = \normlr{\frac{1}{M}\sum_{k = 1}^M (g_{(i, k)}(K) - \nabla C_{j, r}(K))}$, where we denoted the $k$-th component of the minibatch as $g_{(i, k)}(K)$. Recall from \eqref{eqn:zero-order-grad} that this $k$-th component takes the form $$g_{(i, k)}(K) =  C_{j}(K + rU_k; \mathcal{Z}_k^{(i)})\left(\frac{D}{r}\right)U_k.$$ Hence, we have $\norm{g_{(i, k)}(K)} \leq \frac{D}{r}G_\infty$ due to \eqref{eqn:uniformly_bounded_noisy_cost} as $\norm{U_k} = 1$. Let us define $c_{(p,8)} \coloneqq \left(G_\infty +  \lambda\frac{\rho}{D} + \phi \frac{\rho^2}{D}\right)$. We then have 
\begin{align*}
    \norm{g_{(i, k)}(K) - \nabla C_{j, r}(K)} &= \norm{g_{(i, k)}(K) - \nabla C_{j, r}(K) + \nabla C_j(K) - \nabla C_j(K)} \\
   &  \overset{(a)} \leq \norm{g_{(i, k)}(K)} + \norm{\nabla C_{j, r}(K) - \nabla C_j(K)} + \norm{\nabla C_j(K)} \\
    & \overset{(b)}\leq \frac{D}{r}G_\infty + \phi r + \lambda \\
    & = \frac{D}{r}\left( G_\infty + \frac{r^2}{D}\phi  + \frac{r}{D}\lambda \right) \\
    & \overset{(c)} \leq \frac{D}{r}\left( G_\infty + \frac{\rho^2}{D}\phi  + \frac{\rho}{D}\lambda \right) \\
    & = \frac{D}{r}c_{(p, 8)}.
\end{align*}
In the above, $(a)$ follows from the triangle inequality, and $(b)$ follows from the uniform-boundedness of the noisy rollout together with \eqref{eqn:zero-order-grad}, the bounded bias property as shown in \eqref{eqn:bounded_bias}, and the local-Lipschitz property in Lemma~\ref{lem:local_lipschitz}. Finally, $(c)$ follows from using $ r < \rho$. Hence, $g_{(i, k)}(K) - \nabla C_{j, r}(K)$ has a bounded norm and therefore belongs to a class of norm sub-Gaussian random matrices \citep{jin2019short}. Furthermore, it has zero mean due to \eqref{eqn:unbiased_grad}. Therefore, the concentration result follows from a direct application of Corollary 7 from \cite{jin2019short} which provides a Hoeffding-type inequality for norm sub-Gaussian random matrices.   
\end{proof}

\begin{corollary} \label{cor:collaborative_grad_concentration} (\textbf{Concentration of the collaborative zeroth-order gradient estimates}.) Suppose Assumption~\ref{ass:independence} holds. For any system $j \in [H]$, given a policy $K \in \mathcal{G}_j^0$, a smoothing radius $r \in (0, \rho)$, define the collaborative zeroth-order gradient estimate as $G_j(K) = \frac{1}{\abs{\mathcal{M}_j}}\sum_{i \in \mathcal{M}_j}g_i(K)$, where $g_i(K)$ is the $M$-minibatched gradient estimate from agent $i \in \mathcal{M}_j$. Let $\delta' \in (0, 1)$. The following holds with probability at least $1 - \delta'$:
\begin{equation}
\label{eqn:collab_grad_concentration}
    \norm{G_j(K) - \nabla C_{j, r}(K)} \leq \frac{\left(G_\infty +  \lambda\frac{\rho}{D} + \phi \frac{\rho^2}{D}\right)D}{r\sqrt{\abs{\mathcal{M}_j}M}} \sqrt{\log\left(\frac{2D}{\delta'}\right)}.
\end{equation}
\end{corollary}
\begin{proof}
Under Assumption~\ref{ass:independence}, we note that the noise processes for all agents in $\mathcal{M}_j$ are independent. The proof then follows from Lemma~\ref{lem:grad_concentration} as the collaborative zeroth-order gradient estimate can be interpreted as a gradient estimate for system $j$ with a $\abs{\mathcal{M}_{j}}$-fold increased minibatch size.
\end{proof}

\textbf{Note on the problem dependent constants.} In \cite{malik2020derivative}, the values of the constants $\lambda_j, \phi_j, \rho_j, G_\infty^{(j)}$ are first derived locally in terms of the local cost $C_j(K)$, and then, the global parameters are obtained by noting that the local cost is uniformly bounded over the restricted domain as shown in Lemma 9 of \cite{malik2020derivative}. Since we have a different definition of the restricted domain, the values of our parameters vary from the ones provided in \cite{malik2020derivative}. That said, the global parameters in our setting can be derived exactly in the same way as in~\cite{malik2020derivative} by bounding the local cost as $C_j(K) \leq 10 \tilde{\Delta}_0 + C_j(K_j^*)$. 

For convenience, we compile all the relevant notation in Table~\ref{tab:notation_def}. 

\begin{table}[]
\centering
\caption{Relevant notation and definitions}
\vspace{2mm}
\label{tab:notation_def}
\begin{tabular}{ c c }
\hline
 Notation & Definition \\ 
\hline
 $M_l$ & Minibatch size used to estimate zeroth-order gradients in the $l$-th epoch. \\  
 $R_l$ & Number of steps/iterations of the local and global policy optimization in the $l$-th epoch. \\
 $\eta$ & Step size for both local and global policy optimization. \\
 $r_l$ & Smoothing radius used in the zeroth-order gradient estimates in the $l$-th epoch. \\
 $\Delta_l$ & Estimate of $\Delta$ used to cluster the agents in the $l$th epoch. \\
 $\mathcal{N}_i^{(l)}$ & Neighborhood set corresponding to the $i$-th agent in the $l$-th epoch.\\
$X_i^{(l)}$ & Local policy for the $i$-th agent in the $l$-th epoch.\\
$\hat{K}_i^{(l)}$ & Global policy for the $i$-th agent in the $l$-th epoch.\\
 \hline
\end{tabular}
\end{table}

In the main text, we used the notation $c_{(p, \_)}$ to denote the problem-parameter-dependent constants which are defined in the following. 

\begin{equation}
\label{eqn:constants_def}
    \begin{aligned}
        c_{(p, 8)} &= \left(G_\infty +  \lambda\frac{\rho}{D} + \phi \frac{\rho^2}{D}\right) \\
        c_{(p, 9)} &= \frac{12c_{(p, 8)}}{\mu}\left(\max \left\{\sqrt{\phi}, \frac{1}{\rho}\right\}\right)^{2}\\
        c_{(p, 10)} &= \max\left\{\Delta_0^2, 256c_{(p, 9)}^2D^2, c_{(p, 8)}^2D^2\Delta_0^2, 36 G_\infty^2\right\} \\
        c_{(p, 11)} &= \left( \frac{\Delta_0^2}{c_{(p, 10)}\log\left(\frac{8DN}{\delta}\right)}\right) \\
        c_{(p, 12)} &= \frac{4}{\eta \mu} \left( \log\left(\frac{c_{(p, 10)}N\tilde{\Delta}_0^2}{\Delta_0^2}\right) + \log(4)\right) \\
        c_{(p, 13)} &= 4\max\{1, c_{(p, 9)}\}\sqrt{c_{(p, 12)}\log(c_{(p, 11)}T)} \\
        c_{(p, 1)} &= \min\left\{\frac{8}{\mu}, \frac{1}{4\phi}, \frac{\rho}{\lambda + 2 \max\left\{\sqrt{\phi}, \frac{1}{\rho}\right\}}\right\}\\
        c_{(p, 2)} &= \frac{4}{\eta \mu}\\
        c_{(p, 3)} &= \tilde{\Delta}_0\max\{16, 10 c_{(p, 10)}\}\\
        c_{(p, 4)} &= c_{(p, 10)}\\
        c_{(p, 5)} &= \frac{c_{(p, 8)}D}{\phi}\\
        c_{(p, 6)} &= \rho \\
        c_{(p, 7)} &= Dc_{(p, 13)}\\
    \end{aligned}
\end{equation}

Based on the above definitions of the problem-parameter-dependent constants, we provide the values used for the hyperparameters in the $l$-th epoch of the \texttt{PCPO} algorithm in Table~\ref{tab:hyp_vals}.

\begin{table}[t]
\centering
\caption{Hyperparameters with their values in the $l$th epoch}
\vspace{2mm}
\label{tab:hyp_vals}
\renewcommand{\arraystretch}{2}
\begin{tabular}{| c | c |}
\hline
 \textbf{Hyperparameters} & \textbf{Values} \\ 
\hline
 $\Delta_l$ & $\frac{\Delta_0}{2^l}$ \\
 \hline
 $\delta_l$ & $\frac{\delta}{2l^2}$ \\
 \hline
 $\eta$ & $c_{(p, 1)}$ \\
 \hline
 $R_l$ & $c_{(p, 2)} \log\left(\frac{c_{(p, 3)}N}{\Delta_l^2}\right)$ \\
 \hline
 $M_l$ & $\frac{c_{(p, 4)}}{\Delta_l^2}\log\left(\frac{8DNR_l}{\delta_l}\right)$ \\
 \hline
 $\tilde{r}_l$ & $\left(\frac{c_{(p, 5)}}{\sqrt{M_l}} \sqrt{\log\left(\frac{8DNR_l}{\delta_l}\right)}\right)^{1/2}$ \\
 \hline
 $r_l^{(\text{loc})}$ & $\min\{c_{(p, 6)}, \tilde{r}_l\}$ \\
 \hline
 $r_l^{(\text{global})}$ & $\min\left\{c_{(p, 6)}, \frac{\tilde{r}_l}{\abs{\mathcal{N}_i^{(l-1)}}^{1/4}}\right\}$\\
 \hline
\end{tabular}
\end{table}

\newpage
\section{Proof of Theorem~\ref{thm:clustering}}
\label{app:proof_thm1}
In this section, we provide the proof of Theorem~\ref{thm:clustering} which concerns the clustering aspect of the \texttt{PCPO} algorithm. In particular, we show that, with high probability, the true clusters are included in the neighborhood sets for all agents in each epoch, and if epoch $l \geq L = \min\{l \in 1, 2, \ldots : \Delta_l \leq \Delta/2\}$, the neighborhood sets are identical to the clusters. More specifically, we show that for all agents $i \in [N]$, with probability at least $1 - \delta/2$, $\mathcal{M}_{\sigma(i)} \subseteq \mathcal{N}_i^{(l)}$ in every epoch $l$, and if $l \geq L$, then $\mathcal{M}_{\sigma(i)} = \mathcal{N}_i^{(l)}$.

To prove both claims, it suffices to show that with high probability, the estimated cost at a locally optimized policy is concentrated in the $\Delta_l/4$-neighborhood of the optimal cost in every epoch $l$ for all agents $i \in [N]$. To see this, consider a ``good'' event that occurs with probability $1 - \delta/2$ where the following holds for all agents $i \in [N]$ in every epoch $l$ (we will prove that such an event exists later in this section): 
\begin{equation}
|\hat{C}_{\sigma(i)}(X^{(l)}_i) - C_{\sigma(i)}(K^*_{\sigma(i)})| \leq \Delta_{l}/4.
\label{eqn:local_cost_est_proof}
\end{equation}
On this ``good'' event, in what follows, we show that the first claim of Theorem~\ref{thm:clustering} holds. Accordingly, fix an agent $i$ and consider an agent $j \in \mathcal{M}_{\sigma(i)}$. We now show by induction that $j$ belongs to $\mathcal{N}_i^{(l)}$ in every epoch $l$. For the base case of induction, note that since we initialize the neighborhood sets with all agents, $j \in \mathcal{N}_i^{(0)}$. Next, for an epoch $l-1 \geq 1$, let us assume that $j \in \mathcal{N}_i^{(l-1)}$. Since $ C_{\sigma(i)}(K^*_{\sigma(i)}) =  C_{\sigma(j)}(K^*_{\sigma(j)})$ as a consequence of $j \in \mathcal{M}_{\sigma(i)}$, in the $l$-th epoch under the ``good'' event where \eqref{eqn:local_cost_est_proof} holds, we have 
\begin{align}
  |\hat{C}_{\sigma(i)}(X^{(l)}_i) - \hat{C}_{\sigma(j)}(X^{(l)}_j)| &\leq |\hat{C}_{\sigma(i)}(X^{(l)}_i) - C_{\sigma(i)}(K^*_{\sigma(i)})| + |\hat{C}_{\sigma(j)}(X^{(l)}_j) - C_{\sigma(j)}(K^*_{\sigma(j)})| \nonumber\\
  &\leq \Delta_{l}/4 + \Delta_l/4 = \Delta_l/2, \nonumber 
\end{align}
implying that $j \in \mathcal{N}_i^{(l)}$ based on the neighborhood set update rule in \eqref{eqn:neighborhood_update}. Hence, by induction, $j \in \mathcal{N}_i^{(l)}$ in every epoch, therefore establishing the first claim of Theorem~\ref{thm:clustering}. 

Next, we show that on the ``good'' event where \eqref{eqn:local_cost_est_proof} holds for all agents in every epoch, the second claim of Theorem~\ref{thm:clustering} is also true. We prove this claim via contradiction. To proceed, suppose that there exist an epoch $l \geq L$, an agent $i$, and an agent $j \notin \mathcal{M}_{\sigma(i)}$ such that $j \in \mathcal{N}_i^{(l)}$. Then, we have the following in light of the heterogeneity metric defined in \eqref{eqn:het}:
\begin{align*}
    \Delta &\leq  \abslr{C_{\sigma(i)}(K^*_{\sigma(i)}) - C_{\sigma(j)}(K^*_{\sigma(j)})} \\
    &\leq \abslr{C_{\sigma(i)}(K^*_{\sigma(i)}) - \hat{C}_{\sigma(i)}(X_i^{(l)}) + \hat{C}_{\sigma(j)}(X_j^{(l)}) - C_{\sigma(j)}(K^*_{\sigma(j)}) + \hat{C}_{\sigma(i)}(X_i^{(l)}) - \hat{C}_{\sigma(j)}(X_j^{(l)})} \\
    & \leq \abslr{C_{\sigma(i)}(K^*_{\sigma(i)}) - \hat{C}_{\sigma(i)}(X_i^{(l)})} + \abslr{\hat{C}_{\sigma(j)}(X_j^{(l)}) - C_{\sigma(j)}(K^*_{\sigma(j)})} + \abslr{\hat{C}_{\sigma(i)}(X_i^{(l)}) - \hat{C}_{\sigma(j)}(X_j^{(l)})} \\
    & \overset{(a)}{\leq} \Delta_l/4 + \Delta_l/4 + \abslr{\hat{C}_{\sigma(i)}(X_i^{(l)}) - \hat{C}_{\sigma(j)}(X_j^{(l)})} \\
    & \overset{(b)}{\leq} \Delta/4 + \abslr{\hat{C}_{\sigma(i)}(X_i^{(l)}) - \hat{C}_{\sigma(j)}(X_j^{(l)})},
\end{align*}
where $(a)$ holds due to \eqref{eqn:local_cost_est_proof}, and $(b)$ follows as $\Delta_l \leq \Delta/2$ since $l \geq L$. The above set of inequalities imply that $\abslr{\hat{C}_{\sigma(i)}(X_i^{(l)}) - \hat{C}_{\sigma(j)}(X_j^{(l)})} \geq (3/4)\Delta \geq (3/2)\Delta_l$, contradicting our assumption that $j \in \mathcal{N}_i^{(l)}$ which requires that $\abslr{\hat{C}_{\sigma(i)}(X_i^{(l)}) - \hat{C}_{\sigma(j)}(X_j^{(l)})} \leq \Delta_l/2$. Therefore, $\mathcal{N}_i^{(l)} = \mathcal{M}_{\sigma(i)}$ for all $l \geq L$, establishing the second claim of Theorem~\ref{thm:clustering}.

Now, it remains to prove that the ``good'' event where \eqref{eqn:local_cost_est_proof} holds for all agents in every epoch occurs with probability at least $1 - \delta/2$.  To do so, for an agent $i \in [N]$ in epoch $l$, we have $$\abslr{\hat{C}_{\sigma(i)}(X^{(l)}_i) - C_{\sigma(i)}(K^*_{\sigma(i)})} \leq \abslr{C_{\sigma(i)}(X^{(l)}_i) - C_{\sigma(i)}(K^*_{\sigma(i)})} + \abslr{\hat{C}_{\sigma(i)}(X^{(l)}_i) - C_{\sigma(i)}(X^{(l)}_i)}.$$ Therefore, to show \eqref{eqn:local_cost_est_proof}, it suffices to show the following two guarantees for all agents in every epoch: 
\begin{equation}
\label{eqn:proof_thm1_req}
    \begin{aligned}
    (a) ~~~ \abslr{C_{\sigma(i)}(X^{(l)}_i) - C_{\sigma(i)}(K^*_{\sigma(i)})} &\leq \Delta_l/8 \\
    (b) ~~~ \abslr{\hat{C}_{\sigma(i)}(X^{(l)}_i) - C_{\sigma(i)}(X^{(l)}_i)} &\leq \Delta_l/8.
    \end{aligned}
\end{equation}
Now, let us establish the claims in \eqref{eqn:proof_thm1_req} via induction across epochs. Let us assume that for a fixed epoch $l-1 \geq 1$, the claims in \eqref{eqn:proof_thm1_req} hold for all agents $i \in [N]$ in every epoch $ k \leq \{1, 2, \ldots, l - 1\}$, with probability at least $(1 - \sum_{j = 1}^{l-1}\frac{\delta_j}{2})$. Denoting this event as $\mathcal{E}_{l-1}$, in the following, we show that the claims in \eqref{eqn:proof_thm1_req} hold in epoch $l$, $\forall i \in [N]$ with probability at least $(1 - \frac{\delta_{l}}{2})$ conditioned on the event $\mathcal{E}_{l-1}$. 

In the following, fixing an agent $i \in [N]$, we show: (a) $\abslr{C_{\sigma(i)}(X^{(l)}_i) - C_{\sigma(i)}(K^*_{\sigma(i)})} \leq \Delta_l/8$ with probability at least $1 - \delta_l/(4N)$ conditioned on the event $\mathcal{E}_{l-1}$, and (b) $\abslr{\hat{C}_{\sigma(i)}(X^{(l)}_i) - C_{\sigma(i)}(X^{(l)}_i)} \leq \Delta_l/8$ with probability at least $1 - \delta_l/(4N)$ conditioned on the intersection of the events $\mathcal{E}_{l-1}$ and the one where item (a) in~\eqref{eqn:proof_thm1_req} holds. The following lemma provides the convergence of the local policy optimization sub-routine in epoch $l$ which aids in establishing claim (a) from \eqref{eqn:proof_thm1_req}.

\begin{lemma} \label{lem:localPO} (\textbf{Local Policy Optimization}.) For any agent $i \in \mathcal{M}_j$, given a policy $K_0 \in \mathcal{G}_j^0$, let $K_R$ be the output of the $\texttt{localPO}(K_0, M, R, r)$ subroutine with step size $\eta$. Then, for any $\delta' \in (0, 1/R)$, with probability at least $1 - \delta'R$, $K_R \in \mathcal{G}_j^0$ and we have the following:
\begin{equation}
    \label{eqn:local_po_lemma}
    C_j(K_R) - C_j(K^*_j) \leq \left( 1 - \frac{\eta \mu}{4} \right)^R (C_j(K_0) - C_j(K^*_j)) + \left(\frac{c_{(p, 9)}D}{\sqrt{M}} \sqrt{\log\left(\frac{2D}{\delta'}\right)}\right),
\end{equation}
when $\eta = c_{(p, 1)}, M \geq \frac{c_{(p, 4)}}{\Delta_0^2} \log(2D/\delta'), r = \min\{\rho, \tilde{r}\}$, where $\tilde{r} = \left(\frac{c_{(p, 5)}}{\sqrt{M}} \sqrt{\log\left(\frac{2D}{\delta'}\right)}\right)^{1/2}$.
\end{lemma}

The proof of Lemma~\ref{lem:localPO} is provided in Appendix~\ref{app:proof_localPO}. We use Lemma~\ref{lem:localPO} to analyze the local policy optimization step in line 4 of the \texttt{PCPO} algorithm that helps in establishing $\abslr{C_{\sigma(i)}(X^{(l)}_i) - C_{\sigma(i)}(K^*_{\sigma(i)})} \leq \Delta_l/8$ with probability at least $1 - \delta_l/(4N)$. More precisely, the settings for the hyperparameters $(\eta, M_l, r_l^{(\text{loc})}, R_l)$ from Table~\ref{tab:hyp_vals} meet the requirement for the corresponding hyperparameters in Lemma~\ref{lem:localPO}. Furthermore, conditioned on the event $\mathcal{E}_{l-1}$, claim $(a)$ in \eqref{eqn:proof_thm1_req} implies that $X_i^{(l-1)} \in \mathcal{G}_{\sigma(i)}^0$. Hence, the following holds due to Lemma~\ref{lem:localPO} with probability at least $1 - \delta'R_l$ for some $\delta' \in (0, 1/R_l)$:
$$C_{\sigma(i)}(X_i^{(l)}) - C_{\sigma(i)}(K^*_{\sigma(i)}) \leq \underbrace{\left( 1 - \frac{\eta \mu}{4} \right)^{R_l} (C_{\sigma(i)}(X_i^{(l-1)}) - C_{\sigma(i)}(K^*_{\sigma(i)}))}_{s_1} + \underbrace{\left(\frac{c_{(p, 9)}D}{\sqrt{M_l}} \sqrt{\log\left(\frac{2D}{\delta'}\right)}\right)}_{s_2}.$$
Note that $(C_{\sigma(i)}(X_i^{(l-1)}) - C_{\sigma(i)}(K^*_{\sigma(i)})) \leq \Delta_{l-1}/8 \leq \Delta_0 \leq \tilde{\Delta}_0$ as a result of conditioning on the event $\mathcal{E}_{l-1}$ where claim $(a)$ of \eqref{eqn:proof_thm1_req} holds. Hence, from Table~\ref{tab:hyp_vals}, setting $R_l =  c_{(p, 2)} \log\left(\frac{c_{(p, 3)}N}{\Delta_l^2}\right) \geq \frac{4}{\eta \mu}\log \left(\frac{16 \tilde{\Delta}_0}{\Delta_l}\right)$ ensures $s_1 \leq \Delta_l/16$. Similarly, setting $M_l = \frac{c_{(p, 4)}}{\Delta_l^2}\log\left(\frac{2D}{\delta'}\right) \geq \frac{16^2c_{(p, 9)}^2D^2}{\Delta_l^2}\log\left(\frac{2D}{\delta'}\right)$ ensures $s_2 \leq \Delta_l/16$. Finally, setting $\delta' = \delta_l/(4NR_l)$ provides us with $C_{\sigma(i)}(X^{(l)}_i) - C_{\sigma(i)}(K^*_{\sigma(i)}) \leq \Delta_l/8$, and hence $X^{(l)}_i \in \mathcal{G}^{0}_{\sigma(i)}$ with probability at least $(1 - \delta_l/(4N))$, based on Lemma~\ref{lem:localPO}. Let us denote this event by $\tilde{\mathcal{E}}_{(l, 1)}$.

Now, we show that $\abslr{\hat{C}_{\sigma(i)}(X^{(l)}_i) - C_{\sigma(i)}(X^{(l)}_i)} \leq \Delta_l/8$ with probability at least $1 - \delta_l/(4N)$ conditioned on the event $\tilde{\mathcal{E}}_{(l, 1)}$. We have $\hat{C}_{\sigma(i)}(X^{(l)}_i) = \frac{1}{M_l} \sum_{j=1}^{M_l} C_{\sigma(i)}(X_i^{(l)}, \mathcal{Z}_j^{(i)})$ and $\mathbb{E}[C_{\sigma(i)}(X_i^{(l)}, \mathcal{Z}_j^{(i)})] = C_{\sigma(i)}(X^{(l)}_i)$. Furthermore, since on event $\tilde{\mathcal{E}}_{(l, 1)}$, $X_i^{(l)} \in \mathcal{G}_j^0$, we have $C_{\sigma(i)}(X_i^{(l)}, \mathcal{Z}_j^{(i)}) \leq G_\infty$ due to \eqref{eqn:uniformly_bounded_noisy_cost}. Using Hoeffding's inequality, the following holds for all $s \geq 0$:
\begin{align*}
    \mathbb{P}\left(\abslr{\frac{1}{M_l} \sum_{j=1}^{M_l} C_{\sigma(i)}(X_i^{(l)}, \mathcal{Z}_j^{(i)}) - C_{\sigma(i)}(X^{(l)}_i)} \geq s \right) \leq 2\exp\left( \frac{-2s^2M_l}{G_\infty^2}\right).
\end{align*}
Setting $s = \Delta_l/8$, and requiring the failure probability on the R.H.S to be lesser than $\delta_l/(4N)$ leads to the requirement: $M_l \geq \frac{36 G_\infty^2}{\Delta_l^2}\log\left(\frac{8N}{\delta_l}\right)$ which is satisfied by our choice of $M_l$ from Table~\ref{tab:hyp_vals}.  Hence, $\abslr{\hat{C}_{\sigma(i)}(X^{(l)}_i) - C_{\sigma(i)}(X^{(l)}_i)} \leq \Delta_l/8$ with probability at least $1 - \delta_l/(4N)$. Let us denote this event as $\tilde{\mathcal{E}}_{(l, 2)}$. 

Now, to show the two claims in \eqref{eqn:proof_thm1_req}, let us define an event $\tilde{\mathcal{E}}_l = \tilde{\mathcal{E}}_{(l, 1)} \cap \tilde{\mathcal{E}}_{(l, 2)}$. Then, $\mathbb{P}(\tilde{\mathcal{E}}_l \vert \mathcal{E}_{l-1}) = \mathbb{P}(\tilde{\mathcal{E}}_{(l, 2)} \vert \tilde{\mathcal{E}}_{(l, 1)}, \mathcal{E}_{l-1}) \mathbb{P}(\tilde{\mathcal{E}}_{(l, 1)} \vert \mathcal{E}_{l-1}) \geq (1 - \delta_l/(4N))(1 - \delta_l/(4N)) \geq 1 - \delta_l/(2N)$. Therefore, on event $\tilde{\mathcal{E}}_l$, both the guarantees: (a) $\abslr{C_{\sigma(i)}(X^{(l)}_i) - C_{\sigma(i)}(K^*_{\sigma(i)})} \leq \Delta_l/8$ and (b) $\abslr{\hat{C}_{\sigma(i)}(X^{(l)}_i) - C_{\sigma(i)}(X^{(l)}_i)} \leq \Delta_l/8$ hold with probability at least $1 - \delta_l/(2N)$. Union bounding over all the agents, with probability at least $1 - \delta_l/2$, both claims in \eqref{eqn:proof_thm1_req} hold for all agents $i \in [N]$ on the event $\tilde{\mathcal{E}}_l$ after conditioning on the event $\mathcal{E}_{l-1}$.

Finally, defining an event $\mathcal{E}_l = \tilde{\mathcal{E}}_l \cap \mathcal{E}_{l-1}$, we have, 
$$\mathbb{P}(\mathcal{E}_l) = \mathbb{P}(\tilde{\mathcal{E}_l} \vert \mathcal{E}_{l-1}) \mathbb{P}(\mathcal{E}_{l-1}) \geq (1 - \delta_l/2) \left(1 - \sum_{j = 1}^{l-1}\frac{\delta_j}{2} \right) \geq \left(1 - \sum_{j = 1}^{l}\frac{\delta_j}{2} \right).$$ Since $\delta_l = \delta/(2l^2)$, we have $\sum_{j = 1}^{l}\frac{\delta_j}{2} = \sum_{j = 1}^{l}\frac{\delta}{4j^2}\leq \delta/2$. This completes the proof of Theorem~\ref{thm:clustering}. 

\subsection{Proof of Lemma~\ref{lem:localPO}}
\label{app:proof_localPO}
In this section, we prove Lemma~\ref{lem:localPO} which provides the convergence of the \texttt{localPO} subroutine. Fix a system $j \in [H]$ and let $K_t$ denote the controller in the $t$-th iteration of \texttt{localPO} $\forall t = 0, 1, \ldots, R$. Note that the \texttt{localPO} sub-routine proceeds as follows: starting with a controller $K_0 \in \mathcal{G}^0_j$, in every iteration $t$, $K_t$ is updated as $K_{t+1} = K_t - \eta g(K_t)$, where $g(K_t) = \texttt{ZO}(K_t, M, r)$ is the $M$-minibatched zeroth-order gradient estimate with a smoothing radius $r$. We prove the statement via induction. Given the base case $K_0 \in \mathcal{G}_j^0$ and $\delta' \in (0, 1/R)$, let us assume that in the $t$-th iteration the following holds for all $\tau \in \{1, 2, \ldots, t\}$ with probability at least $(1 - \delta't):$ 

\begin{equation}
    \label{eqn:local_po_lemma_induction}
    \begin{aligned}
    K_\tau &\in \mathcal{G}_j^0 \\
    C_j(K_\tau) - C_j(K^*_j) &\leq \left( 1 - \frac{\eta \mu}{4} \right) (C_j(K_{\tau-1}) - C_j(K^*_j)) + \frac{\eta \mu}{4}\left(\frac{c_{(p, 9)}D}{\sqrt{M}} \sqrt{\log\left(\frac{2D}{\delta'}\right)}\right).  
    \end{aligned}
\end{equation}
Let us denote the event where both the claims in \eqref{eqn:local_po_lemma_induction} hold by $E_t$. Now, conditioned on the event $E_t$, in the following, we will show that with probability at least $1 - \delta'$, $K_{t+1} \in \mathcal{G}_j^0$ and $$ C_j(K_{t+1}) - C_j(K^*_j) \leq \left( 1 - \frac{\eta \mu}{4} \right) (C_j(K_{t}) - C_j(K^*_j)) + \frac{\eta \mu}{4}\left(\frac{c_{(p, 9)}D}{\sqrt{M}} \sqrt{\log\left(\frac{2D}{\delta'}\right)}\right).$$

In what follows, we omit the subscript notation $j$ for convenience. Conditioned on the event $E_t$, we begin by analyzing the one-step progress in the $(t+1)$-th iteration of \texttt{localPO}. 

From Lemma~\ref{lem:grad_concentration}, as the event $E_t$ ensures that $K_t \in \mathcal{G}^0$, we have $\norm{g(K_t) - \nabla C_{r}(K_t)} \leq \frac{c_{(p, 8)}D}{r\sqrt{M}} \sqrt{\log\left(\frac{2D}{\delta'}\right)}$ with probability at least $(1 - \delta')$. Let us denote this event by $\tilde{E}_t$. Define $e_t \coloneqq g(K_t) - \nabla C(K_t)$. Conditioned on the event $\tilde{E}_t \cap E_t$, we have
\begin{align}
    \norm{e_t} &= \norm{ g(K_t) - \nabla C_{r}(K_t) + \nabla C_{r}(K_t) - \nabla C(K_t)} \nonumber\\
    &\overset{(a)}{\leq} \norm{ g(K_t) - \nabla C_{r}(K_t)} + \norm{\nabla C_{r}(K_t) - \nabla C(K_t)} \nonumber\\
    &\overset{(b)}{\leq} \frac{c_{(p, 8)}D}{r\sqrt{M}} \sqrt{\log\left(\frac{2D}{\delta'}\right)} + \phi r, \label{eqn:et_bound}
\end{align}
where $(a)$ follows from the triangle inequality and the $(b)$ due to the event $\tilde{E}_t \cap E_t$ and \eqref{eqn:bounded_bias}. Let us define $c_p = \frac{c_{(p, 8)}D}{r\sqrt{M}} \sqrt{\log\left(\frac{2D}{\delta'}\right)} + \phi r$. Set $\tilde{r} = \left(\frac{c_{(p, 8)}D}{\phi\sqrt{M}} \sqrt{\log\left(\frac{2D}{\delta'}\right)}\right)^{1/2}$,   $r = \min\{\rho, \tilde{r}\}$, and define $Z \coloneqq \frac{c_{(p, 8)}D}{\sqrt{M}} \sqrt{\log\left(\frac{2D}{\delta'}\right)}$. Based on the choice $M \geq \frac{c_{(p, 4)}}{\Delta_0^2} \log(2D/\delta')$, we have $Z \leq 1$, yielding the following sequence of bounds on $c_p:$
\begin{align}
    c_p &= \frac{Z}{r} + \phi r \nonumber\\
    &\leq \max \left\{\frac{Z}{\tilde{r}} + \phi \tilde{r}, \frac{Z}{\rho} + \phi \rho \right\} \nonumber\\
    &\overset{(a)}{\leq} \max \left\{2 \sqrt{Z \phi}, \frac{\sqrt{Z}}{\rho} + \phi \tilde{r} \right\} \nonumber\\
    & = \max \left\{2 \sqrt{Z \phi}, \frac{\sqrt{Z}}{\rho} + \sqrt{Z \phi} \right\} \nonumber\\
    & \leq 2 \sqrt{Z} \max \left\{\sqrt{\phi}, \frac{1}{\rho}\right\} \label{eqn:cp_bound1}\\
    & \overset{(b)}{\leq} 2 \max \left\{\sqrt{\phi}, \frac{1}{\rho}\right\} \label{eqn:cp_bound2},
\end{align}
where $(a)$ and $(b)$ follow from $Z \leq 1$. Based on the above, we have
\begin{align*}
    \eta\norm{g(K_t)} &\leq \eta (\norm{\nabla C(K_t)}) + \norm{e_t} \\
    & \overset{(a)}{\leq} \eta(\lambda + c_p )\\
    & \overset{(b)}{\leq} \eta \left(\lambda + 2 \max\left\{\sqrt{\phi}, \frac{1}{\rho}\right\}\right),
\end{align*}
where $(a)$ follows from Lemma~\ref{lem:local_lipschitz} and $(b)$ follows from \eqref{eqn:cp_bound2}. Setting the RHS $\leq \rho$ leads to the requirement $\eta \leq \frac{\rho}{\lambda + 2 \max\left\{\sqrt{\phi}, \frac{1}{\rho}\right\}}$ which is satisfied by setting $\eta = c_{(p, 1)}$. This ensures that $\norm{K_{t+1} - K_t} \leq \rho$. Using the local smoothness property (Lemma~\ref{lem:local_smooth}), we then have
\begin{align*}
    C(K_{t+1}) - C(K_{t}) &\leq \langle \nabla C(K_{t}), K_{t+1} - K_{t} \rangle + \frac{\phi}{2}  \norm{K_{t+1} - K_{t}}^2 \\
    &= -\eta \langle \nabla C(K_{t}), g(K_t) \rangle + \frac{\phi \eta^2}{2} \norm{g(K_t)}^2 \\
    &= -\eta \langle \nabla C(K_{t}), \nabla C(K_{t}) + e_t \rangle + \frac{\phi \eta^2}{2} \norm{\nabla C(K_{t}) + e_t}^2\\
    & \overset{(a)}{\leq} -\eta \norm{\nabla C(K_{t})}^2 -\eta  \langle \nabla C(K_{t}), e_t \rangle + \phi \eta^2 \norm{\nabla C(K_{t})}^2 + \phi \eta^2 \norm{e_t}^2 \\
    & \overset{(b)}{\leq} -\eta(1 - \phi\eta) \norm{\nabla C(K_{t})}^2 + \frac{\eta}{2}\norm{\nabla C(K_{t})}^2 + \frac{\eta}{2}\norm{e_t}^2 + \phi \eta^2 \norm{e_t}^2\\
    & = -\frac{\eta}{2}(1 - 2\phi \eta)\norm{\nabla C(K_{t})}^2 + \frac{\eta}{2}(1 + 2\phi\eta)\norm{e_t}^2 \\
    & \overset{(c)}{\leq} -\frac{\eta}{4}\norm{\nabla C(K_{t})}^2 + \frac{3\eta}{4}c_p^2.
\end{align*}
In the above, we used $ \norm{A + B}^2 \leq 2\norm{A}^2 + 2\norm{B}^2$ in $(a)$, and $-2\langle A, B\rangle \leq \norm{A}^2 + \norm{B}^2$ in $(b)$ where $A$ and $B$ are any matrices in $\mathbb{R}^{m \times n}$. In $(c)$, we used $\eta \leq 1/(4\phi)$ (satisfied by our choice $\eta = c_{(p, 1)}$) and $\norm{e_t} \leq c_p$. Denoting the suboptimality gap as $S_t = C(K_t) - C(K^*)$, and using the PL condition \eqref{eqn:PL_condition} in the above, we obtain the following with probability at least $1 - \delta'$:
\begin{align}
     S_{t+1} &\leq \left( 1 - \frac{\eta \mu}{4} \right) S_t + \frac{3\eta}{4}c_p^2 \nonumber\\
     & \leq \left( 1 - \frac{\eta \mu}{4} \right) S_t + \frac{\eta \mu}{4}\left(\frac{3}{\mu}c_p^2\right). \label{eqn:localPo_lemma_one_step}
\end{align}
On event $E_t$, since we have $K_t \in \mathcal{G}_j^0$, $S_t \leq 10\tilde{\Delta}_0$. Furthermore, due to \eqref{eqn:cp_bound1}, we have $$\frac{3}{\mu}c_p^2 \leq \frac{12}{\mu} \left(\max \left\{\sqrt{\phi}, \frac{1}{\rho}\right\}\right)^2 \frac{c_{(p, 8)}D}{\sqrt{M}} \sqrt{\log\left(\frac{2D}{\delta'}\right)}.$$ Defining $c_{(p, 9)} \coloneqq \frac{12c_{(p, 8)}}{\mu}\left(\max \left\{\sqrt{\phi}, \frac{1}{\rho}\right\}\right)^{2}$ and setting $M \geq \frac{c_{(p, 9)}^2D^2}{\Delta_0^2}\log\left(\frac{2D}{\delta'}\right)$,  which is satisfied by $M \geq \frac{c_{(p, 4)}}{\Delta_0^2} \log(2D/\delta')$ from the statement of Lemma~\ref{lem:localPO}, ensures that $\frac{3}{\mu}c_p^2 \leq \Delta_0 \leq 10 \tilde{\Delta}_0$. Hence, based on \eqref{eqn:localPo_lemma_one_step}, we have $S_{t+1} \leq \left( 1 - \frac{\eta \mu}{4} \right) 10 \tilde{\Delta}_0 + \frac{\eta \mu}{4}10 \tilde{\Delta}_0 \leq 10 \tilde{\Delta}_0.$ Therefore, conditioned on the event $\tilde{E}_t \cap E_t$, $K_{t+1} \in \mathcal{G}_j^0$ and $$ S_{t+1} \leq \left( 1 - \frac{\eta \mu}{4} \right) S_{t} + \frac{\eta \mu}{4}\left(\frac{c_{(p, 9)}D}{\sqrt{M}} \sqrt{\log\left(\frac{2D}{\delta'}\right)}\right).$$ 
Now, let us define $E_{t+1} \coloneqq \tilde{E}_t \cap E_t$. We have $\mathbb{P}(\tilde{E}_t \cap E_t) = \mathbb{P}(\tilde{E}_t \vert E_t) \mathbb{P}(E_t) \geq (1 - \delta')(1 - \delta't) \geq 1 - \delta'(t+1)$. This completes the induction step. To prove the statement of Lemma~\ref{lem:localPO}, since $\eta = c_{(p, 1)} \leq 8/\mu$, for any $R \geq 1$, we can unroll the recursion on the event $E_R$ which occurs with probability $1 - \delta'R$. Doing so, we obtain the following which completes the proof:
\begin{align}
    S_{R} &\leq \left( 1 - \frac{\eta \mu}{4} \right)^{R} S_0 +\sum_{k = 0}^{R-1}\left(1 - \frac{\eta \mu}{4}\right)^k \frac{\eta \mu}{4}\left(\frac{c_{(p, 9)}D}{\sqrt{M}} \sqrt{\log\left(\frac{2D}{\delta'}\right)}\right) \nonumber \\
    &\leq \left(1 - \frac{\eta \mu}{4}\right)^{R} S_0 + \left(\frac{c_{(p, 9)}D}{\sqrt{M}} \sqrt{\log\left(\frac{2D}{\delta'}\right)}\right). \label{eqn:localPo_lemma_unroll}
\end{align}

\newpage
\section{Proof of Theorem \ref{thm:collaborative_optimization}}
\label{app:proof_thm2}
We prove Theorem \ref{thm:collaborative_optimization} by conditioning on the event where the claims in \eqref{eqn:proof_thm1_req} hold for all agents in every epoch. In Appendix~\ref{app:proof_thm1}, we showed that such an event occurs with probability at least $1 - \delta/2$ and let us denote it by $\mathcal{E}_{\text{Thm1}}$. Furthermore, under this event, the claims of Theorem~\ref{thm:clustering} hold as shown in Appendix~\ref{app:proof_thm1}. In particular, the true clusters are always contained in the neighborhood sets and correct clustering takes place at the latest during the $L$-th epoch, ensuring that the agents collaborate solely within their own clusters from the $(L+1)$-th epoch onward. With that in mind, we consider the following approach to prove Theorem~\ref{thm:collaborative_optimization}. First, we take for granted that the last epoch occurs after the correct clustering takes place, i.e, $\bar{L} > L$, and later show that this is indeed true if the total number of rollouts $T \geq \tilde{O}(1/\Delta^2)$. Next, we show that for any $l > L$ (note that at least one such epoch exists in light of $\bar{L} > L$,) the policy at the start of the collaborative policy optimization remains in the corresponding restricted domain with high probability, i.e, $\hat{K}_i^{(l-1)} \in \mathcal{G}^0_{\sigma(i)}$. Then, we focus on the last epoch $\bar{L}$ and provide the convergence guarantee, and finally conclude by analyzing the number of rollouts needed to ensure $\bar{L} > L$.

Conditioned on the event $\mathcal{E}_{\text{Thm1}}$, we follow an induction based argument to show that $\hat{K}_i^{(l-1)} \in \mathcal{G}_{\sigma(i)}^0$ for all agents $i \in [N]$ in every epoch $l > L$. Let us define $\tilde{L}$ as the first epoch where correct clustering takes place. Due to Theorem~\ref{thm:clustering}, since the correct clustering takes place at the latest during the $L$th epoch, $\tilde{L} \leq L$, and moreover, since the neighborhood sets are sequentially pruned with no new agents getting added to the neighborhood sets, we have $\mathcal{M}_{\sigma(i)} = \mathcal{N}_i^{(l)}$ for all agents in every epoch $l \geq \tilde{L}$. Furthermore, $\tilde{L}$ being the first epoch where correct clustering takes place, $\mathcal{N}_i^{\tilde{L}-1} \neq \mathcal{N}_i^{\tilde{L}} = \mathcal{M}_{\sigma(i)}$ for some agent $i \in [N]$, hence causing reinitializtion as shown in \eqref{eqn:reinitialization}. After this reinitialization, the global sequences for all the agents are updated by collaborating within their respective clusters, and hence the global sequences for two agents within a cluster evolve identically in light of \eqref{eqn:averaged_grad}. In other words, for all agents $i, j$, if $\mathcal{M}_{\sigma(i)} = \mathcal{M}_{\sigma(j)}$, then for all $l \geq \tilde{L}$, we have the following on the event $\mathcal{E}_{\text{Thm1}}$: 
\begin{equation}
\label{eqn:correct_cluster_conseq}
    \begin{aligned}
        \hat{K}_i^{(l)} &= \hat{K}_j^{(l)} \\
       \mathcal{M}_{\sigma(i)} = \mathcal{N}_i^{(l)} &= \mathcal{N}_j^{(l)} = \mathcal{M}_{\sigma(j)} 
    \end{aligned}
\end{equation}
Taking this into account, we show that $\hat{K}_i^{(l-1)} \in \mathcal{G}_{\sigma(i)}^0$ for all agents $i \in [N]$ in every epoch $l > \tilde{L}$ via induction across epochs. For the base case $l = \tilde{L} + 1$, as a consequence of reinitialization during the $\tilde{L}$th epoch, and since $X_i^{(\tilde{L})} \in \mathcal{G}_{\sigma(i)}^0$ for all agents $i \in [N]$ as a result of conditioning on the event $\mathcal{E}_{\text{Thm1}}$, we have $\hat{K}_i^{(\tilde{L})} \in \mathcal{G}_{\sigma(i)}^0$ for all agents $i \in [N]$. Let us assume that for an epoch $l \geq \tilde{L} + 1$, with probability at least $(1 - \sum_{j = 1}^{l-1}\frac{\delta_j}{4})$, we have $\hat{K}_i^{(t-1)} \in \mathcal{G}_{\sigma(i)}^0$ for all agents $i \in [N]$ and for all $t \in \{\tilde{L} + 1, \tilde{L}+2, \dots, l\}$. With a slight abuse of notation, let us denote this event by $E_{l-1}$. Next, conditioned on the event $\mathcal{E}_{\text{Thm1}} \cap E_{l-1}$, we show that $\hat{K}_i^{(l)} \in \mathcal{G}_{\sigma(i)}^0$ for all agents $i \in [N]$ with probability at least $1 - \delta_l/4$.

In what follows we fix an agent $i \in [N]$ and omit the notation $i$ and $\sigma(i)$ for convenience. Given that $K^{(l-1)} \in \mathcal{G}^0$ on the event $\mathcal{E}_{\text{Thm1}} \cap E_{l-1}$, we focus on analyzing the iterates $\{Y^{(k)}\}_{0\leq k < R_l}$ in the $l$th epoch. The iterates are updated as follows: $Y^{(k+1)} = Y^{(k)} - \eta G(Y^{(k)})$ with $Y^{(0)} = \hat{K}^{(l-1)}$, where we used $G(Y^{(k)})$ to denote the averaged zeroth-order gradient estimate as shown in \eqref{eqn:averaged_grad} in the $k$th iteration. Note that in the light of Assumption~\ref{ass:independence}, the averaged gradient estimate is an unbiased estimate of $\nabla C_r(Y^{(k)})$ with an effective minibatch size of $M_l \abs{\mathcal{M}}$. Therefore, we follow an approach similar to the one from the proof of Lemma~\ref{lem:localPO} in Appendix~\ref{app:proof_localPO} to analyze the one-step progress and to show that $Y^{(R_l)} = \hat{K}^l \in \mathcal{G}^0$. More specifically, we follow an induction based approach across iterations and establish one-step recursion similar to \eqref{eqn:local_po_lemma_induction} and finally unroll the recursion to obtain something similar to \eqref{eqn:localPo_lemma_unroll}. However, a key difference arises from the fact that the second term in the RHS of both \eqref{eqn:local_po_lemma_induction} and \eqref{eqn:localPo_lemma_unroll} will now enjoy a \emph{variance reduction} effect due to collaboration in light of Assumption~\ref{ass:independence} as shown in Corollary~\ref{cor:collaborative_grad_concentration}.

In particular, following the induction approach from the proof of Lemma~\ref{lem:localPO} in Appendix~\ref{app:proof_localPO}, in the $k$th iteration, we have the following concentration with probability at least $(1 - \delta')$ after conditioning on the event where the previous iterations satisfy similar guarantees as in \eqref{eqn:local_po_lemma_induction}:
$$\norm{G(Y^{(k)}) - \nabla C_{r}(Y^{(k)})} \leq \frac{c_{(p, 8)}D}{r\sqrt{M_{l}\abs{\mathcal{M}}}} \sqrt{\log\left(\frac{2D}{\delta'}\right)}.$$ Conditioned on the event where the gradient estimate is concentrated as above, and defining $e_k \coloneqq G(Y^{(k)}) - \nabla C(Y^{(k)})$, we have $\norm{e_{k}} \leq \frac{c_{(p, 8)}D}{r\sqrt{M_{l}\abs{\mathcal{M}}}} \sqrt{\log\left(\frac{2D}{\delta'}\right)} + \phi r$ following the arguments up to \eqref{eqn:et_bound}. Now, let us define $c_p = \frac{c_{(p, 8)}D}{r\sqrt{M_{l}\abs{\mathcal{M}}}} \sqrt{\log\left(\frac{2D}{\delta'}\right)} + \phi r$. Setting $\tilde{r} = \left(\frac{c_{(p, 8)}D}{\phi\sqrt{M_{l}}} \sqrt{\log\left(\frac{2D}{\delta'}\right)}\right)^{1/2}$ and $r = \min\{\rho, \frac{\tilde{r}}{\abs{\mathcal{M}}^{1/4}}\}$, we obtain the following bound on $c_p$ provided $Z \coloneqq \frac{c_{(p, 8)}D}{\sqrt{M_{l}}} \sqrt{\log\left(\frac{2D}{\delta'}\right)} \leq 1$ which is ensured by our setting for $M_l$ in Table~\ref{tab:hyp_vals}. 
\begin{align}
    c_p &= \frac{Z}{r\sqrt{\abs{\mathcal{M}}}} + \phi r \nonumber\\
    &\leq \max \left\{\frac{Z}{\tilde{r}\abs{\mathcal{M}}^{1/4}} + \phi \frac{\tilde{r}}{\abs{\mathcal{M}}^{1/4}}, \frac{Z}{\rho\abs{\mathcal{M}}^{1/2}} + \phi \rho \right\} \nonumber\\
    &\leq \max \left\{2 \frac{\sqrt{Z \phi}}{\abs{\mathcal{M}}^{1/4}}, \frac{\sqrt{Z}}{\rho\abs{\mathcal{M}}^{1/4}} + \phi \frac{\tilde{r}}{\abs{\mathcal{M}}^{1/4}} \right\} \nonumber\\
    & \leq 2 \frac{\sqrt{Z}}{\abs{\mathcal{M}}^{1/4}} \max \left\{\sqrt{\phi}, \frac{1}{\rho}\right\} \label{eqn:cp_bound3}\\
    & \leq 2 \max \left\{\sqrt{\phi}, \frac{1}{\rho}\right\} \label{eqn:cp_bound4}.
\end{align}
Based on the above, we choose $\eta = c_{(p, 1)} \leq \frac{\rho}{\lambda + 2 \max\left\{\sqrt{\phi}, \frac{1}{\rho}\right\}}$ to ensure that $\norm{Y^{{(k+1)}} - Y^{(k)}} \leq \rho$. Defining $S_k = C(Y^{(k)}) - C(K^*)$ and following the analysis from Appendix~\ref{app:proof_localPO} up to \eqref{eqn:localPo_lemma_one_step} and using the bound on $c_p$ from \eqref{eqn:cp_bound3}, we obtain

$$S_{k+1} \leq \underbrace{\left( 1 - \frac{\eta \mu}{4} \right) S_{k}}_{s_1} + \underbrace{\frac{\eta \mu}{4}\left(\frac{c_{(p, 9)}D}{\sqrt{M_l\abs{\mathcal{M}}}} \sqrt{\log\left(\frac{2D}{\delta'}\right)}\right)}_{s_2},$$ with probability at least $1 - \delta'$. Note that since $\abs{\mathcal{M}} \geq 1$, the term $s_2$ is not greater than the corresponding term from \eqref{eqn:local_po_lemma_induction}, and hence $s_2 \leq \frac{\eta \mu}{4}\Delta_0 \leq \frac{\eta \mu}{4}10 \tilde{\Delta}_0$. Meanwhile, the term $s_1 \leq \left( 1 - \frac{\eta \mu}{4} \right) 10 \tilde{\Delta}_0$ as we have conditioned on the event where $Y^{(k)} \in \mathcal{G}^0$ similar to the proof of Lemma~\ref{lem:localPO}. This ensures that $Y^{(k+1)} \in \mathcal{G}^0$. Now, unrolling the recursion, we have with probability at least $1 - \delta'R_l$, $Y^{(R_l)}=\hat{K}^{(l)} \in \mathcal{G}^0$ and the following: 
\begin{equation}
\label{eqn:thm2_unroll}
  C(\hat{K}^{(l)}) - C(K^*) \leq \left(1 - \frac{\eta \mu}{4}\right)^{R_{l}} (C(\hat{K}^{(l-1)}) - C(K^*)) + \left(\frac{c_{(p, 9)}D}{\sqrt{M_l\abs{\mathcal{M}}}} \sqrt{\log\left(\frac{2D}{\delta'}\right)}\right).  
\end{equation}
Setting $\delta' = \delta_l/(4R_lN)$ and applying an union bound over all agents, we have the above guarantee for all agents with probability at least $1 - \delta_l/4$. 

Let us denote this event by $\tilde{E}_l$. Defining $E_l = \tilde{E}_l \cap E_{l-1}$, we have the following: $$\mathbb{P}(E_l \vert \mathcal{E}_{\text{Thm1}}) = \mathbb{P}(\tilde{E_l} \vert E_{l-1}, \mathcal{E}_{\text{Thm1}}) \mathbb{P}(E_{l-1}\vert \mathcal{E}_{\text{Thm1}}) \geq (1 - \delta_l/4) \left(1 - \sum_{j = 1}^{l-1}\frac{\delta_j}{4} \right) \geq \left(1 - \sum_{j = 1}^{l}\frac{\delta_j}{4} \right).$$  This completes the induction argument. Hence, we have established that $\hat{K}_i^{(l-1)} \in \mathcal{G}_{\sigma(i)}^0$ and that \eqref{eqn:thm2_unroll} holds for all agents $i \in [N]$ in every epoch $l \geq \tilde{L}$ with probability at least $\left(1 - \sum_{j = 1}^{l}\frac{\delta_j}{4} \right)$. 

Next, we analyze the final convergence guarantee in the last epoch $\bar{L}$. Conditioned on the event $E_{\bar{L}-1} \cap \mathcal{E}_{\text{Thm1}}$, we obtain \eqref{eqn:thm2_unroll} as shown in the following with probability at least $1 - \delta_{\bar{L}}/4$ for all agents $i \in [N]$:
\begin{equation*}
\begin{aligned}
  C_{\sigma(i)} (\hat{K}_i^{(\bar{L})}) - C_{\sigma(i)}(K^*_{\sigma(i)}) &\leq \underbrace{\left(1 - \frac{\eta \mu}{4}\right)^{R_{\bar{L}}} (C_{\sigma(i)} (\hat{K}_i^{(\bar{L}-1)}) - C_{\sigma(i)}(K^*_{\sigma(i)}))}_{s_1} \\
  & \quad + \underbrace{\left(\frac{c_{(p, 9)}D}{\sqrt{M_{\bar{L}}\abs{\mathcal{M}}}} \sqrt{\log\left(\frac{8DNR_{\bar{L}}}{\delta_{\bar{L}}}\right)}\right)}_{s_2}.    
\end{aligned}
\end{equation*}
In the above, the settings of $R_{\bar{L}}$ and $M_{\bar{L}}$ from Table~\ref{tab:hyp_vals} ensures the following: from $R_{\bar{L}} \geq \frac{4}{\eta \mu}\log\left(\frac{c_{(p, 4)}N 10\tilde{\Delta}_0}{\Delta_{\bar{L}}^2}\right)$, note that the term $$s_1 \leq \exp\left(-\frac{\eta \mu R_{\bar{L}}}{4}\right)S_0 \leq \exp\left(-\frac{\eta \mu R_{\bar{L}}}{4}\right)10\tilde{\Delta}_0 \leq \frac{\Delta_{\bar{L}}^2}{c_{(p, 4)}N}.$$ Using $M_{\bar{L}} = \frac{c_{(p, 4)}}{\Delta_{\bar{L}}^2}\log\left(\frac{8DNR_{\bar{L}}}{\delta_{\bar{L}}}\right)$ in the above, we have $s_1 \leq \frac{\log\left(\frac{8DNR_{\bar{L}}}{\delta_{\bar{L}}}\right)}{M_{\bar{L}}N}$. Furthermore, since $M_{\bar{L}} = \frac{c_{(p, 4)}}{\Delta_{\bar{L}}^2}\log\left(\frac{8DNR_{\bar{L}}}{\delta_{\bar{L}}}\right)\geq \log\left(\frac{8DNR_{\bar{L}}}{\delta_{\bar{L}}}\right)$, we have $$s_1 \leq \frac{\sqrt{\log\left(\frac{8DNR_{\bar{L}}}{\delta_{\bar{L}}}\right)}}{\sqrt{M_{\bar{L}}}N} \leq \frac{\sqrt{\log\left(\frac{8DNR_{\bar{L}}}{\delta_{\bar{L}}}\right)}}{\sqrt{M_{\bar{L}}N}} \leq \frac{\sqrt{\log\left(\frac{8DNR_{\bar{L}}}{\delta_{\bar{L}}}\right)}}{\sqrt{M_{\bar{L}}\abs{\mathcal{M}}}}.$$ Together with the term $s_2$ we obtain the following with probability at least $1 - \delta_{\bar{L}}/4$ for all agents $i \in [N]$:
\begin{equation}
\label{eqn:proof_thm2_1}
  C_{\sigma(i)} (\hat{K}_i^{(\bar{L})}) - C_{\sigma(i)}(K^*_{\sigma(i)}) \leq \frac{2 \max\{1, c_{(p, 9)}D\}}{\sqrt{M_{\bar{L}}\abs{\mathcal{M}_{\sigma(i)}}}} \sqrt{\log\left(\frac{8DNR_{\bar{L}}}{\delta}\right)}.  
\end{equation}

The above holds on the event $E_{\bar{L}}$ conditioned on the event $\mathcal{E}_{\text{Thm1}}$. We have $\mathbb{P}(E_{\bar{L}} \vert \mathcal{E}_{\text{Thm1}}) \geq \left(1 - \sum_{j = 1}^{\bar{L}}\frac{\delta_j}{4} \right) = \left(1 - \sum_{j = 1}^{\bar{L}}\frac{\delta}{8j^2} \right)\geq 1 - \delta/4$. Therefore, $$\mathbb{P}(\mathcal{E}_{E_{\bar{L}}} \cap \mathcal{E}_{\text{Thm1}}) = \mathbb{P}(\mathcal{E}_{E_{\bar{L}}} \vert \mathcal{E}_{\text{Thm1}}) \mathbb{P}(\mathcal{E}_{\text{Thm1}}) \geq (1 - \delta/4)(1 - \delta/2) \geq 1 - (\delta/4 + \delta/2) \geq 1 - \delta.$$

Note that the guarantee in \eqref{eqn:proof_thm2_1} provides a rate $\tilde{O}\left(1/\sqrt{M_{\bar{L}}\abs{\mathcal{M}_{\sigma(i)}}}\right)$. It remains to show that $M_{\bar{L}} = \tilde{\Omega}(T)$. Since $\Delta_l = \Delta_0/4^l$ and $R_{l} \geq 1$ for all $l \in \{1, 2, \ldots, \bar{L}\}$, consider the following as $M_l \leq T$
\begin{align}
    M_{l} \leq T \implies 4^{l} &\leq \frac{T \Delta_0^2}{c_{(p, 10)}\log\left(\frac{8DN}{\delta}\right)} \nonumber\\
    \implies l &\leq \log \left( \frac{T \Delta_0^2}{c_{(p, 10)}\log\left(\frac{8DN}{\delta}\right)}\right) \nonumber\\
    &= \log(c_{(p, 11)}T), \label{eqn:l_bound}
\end{align}
where we defined $c_{(p, 11)} \coloneqq \left( \frac{\Delta_0^2}{c_{(p, 10)}\log\left(\frac{8DN}{\delta}\right)}\right)$. Now, we use the upper bound on $l$ to bound $R_l$ as follows:
\begin{align*}
    R_l &= \frac{4}{\eta \mu}\left( \log\left(\frac{c_{(p, 10)}N\tilde{\Delta}_0^2}{\Delta_0^2}\right) + l\log(4)\right) \\
    & \overset{(a)}{\leq} l \left(\frac{4}{\eta \mu} \left( \log\left(\frac{c_{(p, 10)}N\tilde{\Delta}_0^2}{\Delta_0^2}\right) + \log(4)\right) \right) \\
    &\overset{(b)}{\leq} c_{(p, 12)}\log(c_{(p, 11)}T),
\end{align*}
where $(a)$ follows as $l \geq 1$, and $(b)$ follows from \eqref{eqn:l_bound} with $c_{(p, 12)} \coloneqq \frac{4}{\eta \mu} \left( \log\left(\frac{c_{(p, 10)}N\tilde{\Delta}_0^2}{\Delta_0^2}\right) + \log(4)\right)$. Therefore, the overall sample complexity has the following bound:
\begin{equation}
\label{eqn:T_bound}
\left.
\begin{aligned}
T &= \sum_{l = 1}^{\bar{L}} (2M_lR_l + M_l) \\
&\overset{(a)}{\leq} \sum_{l = 1}^{\bar{L}} (3M_lR_l) \\
& \overset{(b)}{\leq} \sum_{l = 1}^{\bar{L}} 3c_{(p, 12)}\log(c_{(p, 11)}T) \left(\frac{c_{(p, 10)}\log\left(\frac{8DNT}{\delta}\right)}{\Delta_0^2}\right) 4^l \\
& = 4c_{(p, 12)}\log(c_{(p, 11)}T)\left(\frac{c_{(p, 10)}\log\left(\frac{8DNT}{\delta}\right)}{\Delta_0^2}\right) 4^{\bar{L}} \\
&\implies \frac{T}{4c_{(p, 12)}\log(c_{(p, 11)}T)\left(\frac{c_{(p, 10)}\log\left(\frac{8DNT}{\delta}\right)}{\Delta_0^2}\right)} \leq 4^{\bar{L}}. 
\end{aligned}    ~~~ \right\}
\end{equation}

In the above, $(a)$ follows as $R_l \geq 1$ and $(b)$ as we used $R_l \leq T$ in $M_l$. Using the lower bound on $4^{\bar{L}}$ as obtained above in $M_{\bar{L}}$, we obtain $M_{\bar{L}} \geq \frac{T\log\left(\frac{8DNR_{\bar{L}}}{\delta}\right)}{4c_{(p, 12)}\log(c_{(p, 11)}T) \log\left(\frac{8DNT}{\delta}\right)}$. Using this bound in \eqref{eqn:proof_thm2_1}, we have the following with probability $1 - \delta$:
\begin{align}
  C_{\sigma(i)} (\hat{K}_i^{(\bar{L})}) - C_{\sigma(i)}(K^*_{\sigma(i)}) &\leq \frac{4 D\max\{1, c_{(p, 9)}\}\sqrt{c_{(p, 12)}\log(c_{(p, 11)}T) \log\left(\frac{8DNT}{\delta}\right)}}{\sqrt{T\abs{\mathcal{M}_{\sigma(i)}}}} \nonumber\\
  & = \frac{Dc_{(p, 13)}\sqrt{ \log\left(\frac{8DNT}{\delta}\right)}}{\sqrt{T\abs{\mathcal{M}_{\sigma(i)}}}}, \label{eqn:proof_thm2_2}
\end{align}
where we defined $c_{(p, 13)} \coloneqq 4\max\{1, c_{(p, 9)}\}\sqrt{c_{(p, 12)}\log(c_{(p, 11)}T)}$. 

Finally, it remains to show that $\bar{L} > L$ when $T \geq \tilde{O}(1/\Delta^2)$. To ensure, $\bar{L} > L$, consider the number of rollouts required up to $(L+1)$th epoch. From \eqref{eqn:T_bound} with the summation from 1 to $L+1$ we have:
\begin{align*}
\sum_{l = 1}^{L+1} (2M_lR_l + M_l) & \leq 4c_{(p, 12)}\log(c_{(p, 11)}T)\left(\frac{c_{(p, 10)}\log\left(\frac{8DNT}{\delta}\right)}{\Delta_0^2}\right) 4^{L+1}.
\end{align*}
In the above, setting $T \geq $ RHS to ensure $\bar{L} > L$, we have
\begin{align*}
    T  &\geq 4c_{(p, 12)}\log(c_{(p, 11)}T)\left(\frac{c_{(p, 10)}\log\left(\frac{8DNT}{\delta}\right)}{\Delta_0^2}\right)4^{L+1} \\
    &\overset{(a)}{\geq} 4c_{(p, 12)}\log(c_{(p, 11)}T)\left(\frac{c_{(p, 10)}\log\left(\frac{8DNT}{\delta}\right)}{\Delta_0^2}\right)\left(\frac{\Delta_0}{\Delta}\right)^2. 
\end{align*}
In the above, since $L = \min\{l \in 1, 2, \ldots : \Delta_l \leq \Delta/2\}$, $(a)$ follows from the fact that $\Delta_{L+1} = \Delta_0/(2^{L+1}) \leq \Delta/2$. Hence, when $T \geq \tilde{O}(1/\Delta^2)$, we have $\bar{L} > L$. This completes the proof of Theorem~\ref{thm:collaborative_optimization}.

\section{Proof of Corollary \ref{cor:log_complexity}}
\label{app:proof_cor3}
In this section, we analyze the total communication complexity of the \texttt{PCPO} algorithm. In every epoch, each agent communicates with the server once in every iteration of the collaborative policy optimization subroutine, and once to send the local policy to update the neighborhood sets. Hence, the overall communication complexity is $ \sum_{l=1}^{\bar{L}}(R_l + 1).$ Note that $ \sum_{l=1}^{\bar{L}}(R_l + 1) \leq (R_{\bar{L}} + 1) \bar{L}$ since $R_{\bar{L}} = c_{(p, 2)}\log\left(\frac{2^{\bar{L}}c_{(p, 3)}N}{\Delta_{0}^2}\right) \geq c_{(p, 2)}\log\left(\frac{2^lc_{(p, 3)}N}{\Delta_0^2}\right) = R_l$. 

From \eqref{eqn:T_bound}, $\bar{L}$ is logarithmic in $T$. Furthermore, $R_{\bar{L}} = c_{(p, 2)}\log\left(\frac{2^{\bar{L}}c_{(p, 3)}N}{\Delta_0^2}\right)$ is logarithmic in the number of agents $N$ and $T$. Finally, since $T \geq \tilde{O}(1/\Delta^2)$, the overall communication complexity is logarithmic in $T$, $N$ and $1/\Delta$.

\end{document}